\documentclass[11pt]{article}

\usepackage{scicite}
\usepackage{times}

\title{Dynamics of Local Elasticity During Training of Neural Nets}

\usepackage{graphicx} 
\usepackage{caption,subcaption}

\usepackage[
style=alphabetic,
sorting=ynt
]{biblatex}
\usepackage[shortlabels]{enumitem}

\usepackage{amsmath,amsthm,amssymb,amsxtra,mathtools}
\usepackage{dsfont}
\usepackage{array}
\usepackage{graphicx}
\usepackage{algorithm}
\usepackage{algpseudocode}
\usepackage{physics}
\usepackage{MnSymbol,wasysym}
\usepackage{tikzsymbols}
\usepackage{array,tabularx,multirow}
\usepackage{makecell}
\usepackage{array, mathrsfs}
\usepackage[shortlabels]{enumitem}
\usepackage{verbatim}
\usepackage{parskip}
\usepackage{color}
\definecolor{clemson-orange}{RGB}{234,106,32}
\definecolor{chicago-maroon}{RGB}{128,0,0}
\definecolor{cincinnati-red}{RGB}{190,0,0}
\definecolor{soft-cyan}{RGB}{68,85,90}
\definecolor{firebrick}{RGB}{178,34,34}
\definecolor{crimson}{RGB}{220,20,60}
\definecolor{cerrulean}{rgb}{0.165,0.322,0.745}
\definecolor{jaam}{rgb}{0.45,0.0,0.45}

\usepackage{hyperref}
\hypersetup{
  colorlinks   = true,    
  urlcolor     = jaam,    
  linkcolor    = firebrick,    
  citecolor    = blue      
}

\usepackage{parskip}

\usepackage{thmtools}
\declaretheoremstyle[
    headfont=\bfseries, 
    bodyfont=\normalfont\itshape, spaceabove=10pt,
    spacebelow=10pt]{mystyle}
\newtheorem{theorem}{Theorem}[section]
\newtheorem{lemma}[theorem]{Lemma}

\newtheorem{definition}{Definition}
\newtheorem*{remark}{Remark}

\newtheorem{hypothesis}[theorem]{Hypothesis}

\newif\ifsolutions \solutionstrue

\def\final{0}
\newcommand{\reviewer}[3]{
  \expandafter\newcommand\csname #1\endcsname[1]{
    \ifthenelse{\equal{\final}{1}} {
      \textcolor{#3}{}
    } {
      \textcolor{#3}{\begin{center} \textbf{#2} ##1 \end{center}}
    }
  }
}

\reviewer{anirbit}{}{jaam}
\newcolumntype{M}[1]{>{\centering\arraybackslash}m{#1}}

\DeclareMathOperator*{\diag}{diag}

\renewcommand{\ip}[2]{\left\langle#1,#2\right\rangle}

\newcommand{\net}{f_{\mathsf{NN}}}
\newcommand{\ind}[1]{\mathds{1}_{\left\{#1\right\}}}
\newcommand{\relu}{\mathop{\mathrm{ReLU}}}

\newcommand{\distas}{\mathbin{\sim}}

\newcommand{\bbeta}{\boldsymbol{\beta}}
\newcommand{\funcf}{\boldsymbol{f}}
\newcommand{\funcP}{\boldsymbol{P}}

\newcommand{\srel}{S_{\rm rel}}

\def\ve#1{\mathchoice{\mbox{\boldmath$\displaystyle\bf#1$}}
{\mbox{\boldmath$\textstyle\bf#1$}}
{\mbox{\boldmath$\scriptstyle\bf#1$}}
{\mbox{\boldmath$\scriptscriptstyle\bf#1$}}}

\newcommand{\x}{{\ve x}}
\newcommand{\y}{{\ve y}}
\newcommand{\z}{{\ve z}}

\newcommand{\g}{{\ve g}}
\newcommand{\e}{{\ve e}}
\renewcommand{\u}{{\ve u}}

\newcommand{\p}{{\ve p}}
\newcommand{\q}{{\ve q}}

\newcommand{\w}{{\ve w}}

\newcommand{\h}{{\ve h}}

\newcommand{\I}{\textrm{I}}

\newcommand{\W}{\textrm{W}}
\newcommand{\C}{\textrm{C}}

\newcommand{\E}{\mathbb{E}}

\newcommand{\bb}{\mathbb}
\newcommand{\R}{\bb R}
\newcommand{\Z}{{\bb Z}}

\newcommand{\M}{\mathbf{M}}
\renewcommand{\C}{\mathcal C}

\newcommand{\T}{\mathbf{T}}
\newcommand{\KL}{\mathbf{KL}}

\newcommand{\cN}{{\bf \mathcal{N}}}

\newcommand{\cL}{{\bf \mathcal{L}}}

\makeatletter
\def\@fnsymbol#1{\ensuremath{\ifcase#1\or \dagger\or \ddagger\or
   \mathsection\or \mathparagraph\or \|\or **\or \dagger\dagger
   \or \ddagger\ddagger \else\@ctrerr\fi}}
    \makeatother

\author{Soham Dan$^*$ \thanks{Majority of the work done while at the Department of Computer Science ~\& Information Science, UPenn.}\\
IBM Research, United States\\
\texttt{\small soham.dan@ibm.com}
\and
Anirbit Mukherjee$^*$ \thanks{Majority of the work done while at Wharton, the Department of Statistics ~\& Data Science,  UPenn.}\\
Department of Computer Science\\
The University of Manchester\\
\texttt{\small anirbit.mukherjee@manchester.ac.uk}
\and
Avirup Das\\
Department of Computer Science\\
The University of Manchester\\
\texttt{\small avirup.das@postgrad.manchester.ac.uk}
\and
Phanideep Gampa\\
Amazon, United States\\
\texttt{\small gampa.phanideep.mat15@iitbhu.ac.in}
}

\begin{document}

\maketitle
\def\thefootnote{*}\footnotetext{These authors contributed equally to this work.}
\def\thefootnote{\arabic{footnote}}

\begin{abstract}
In the recent past a property of neural training trajectories in weight-space had been isolated, that of ``local elasticity" (denoted as $\srel$). Local elasticity attempts to quantify the propagation of the influence of a sampled data point on the prediction at another data. In this work, we embark on a comprehensive study of the existing notion of $\srel$ as well as propose a new definition that addresses the limitations that we point out for the original definition in the classification setting. On various state-of-the-art neural network training on SVHN, CIFAR-10 and CIFAR-100 we demonstrate how our new proposal of $\srel$, as opposed to the original definition, much more sharply detects the property of the weight updates preferring to make prediction changes within the same class as the sampled data. 

In neural regression experiments we demonstrate that the original $\srel$ reveals a $2-$phase behavior -- that the training proceeds via an initial \textit{elastic} phase when $\srel$ changes rapidly and an eventual \textit{inelastic} phase when $\srel$ remains large. We show that some of these properties can be analytically reproduced in various examples of doing regression via gradient flows on model predictor classes. 

\end{abstract}


\section{Introduction}
\label{sec:introduction}
In recent times, there has been a surge of interest in using neural networks for complex artificial intelligence tasks. Human world champions of classic hard board games have famously been defeated by neural network based approaches, \cite{ silver2017mastering,silver2018general,schrittwieser2020mastering}. The techniques developed for such demonstrations have had deep impacts on various traditional scientific pursuits, for example on planning chemical syntheses \cite{segler2018planning}, drug discovery \cite{chen2018rise} and steering a quantum system towards desired dynamics \cite{dalgaard2020global}. Almost all of these successes depend on the powerful generalization abilities of a deep neural networks, despite its enormous capacity, and we are led to believe that this mystery is intimately tied to the dynamics of training the network. 

\cite{poggio2017theory,zhang2017theory} surveys the myriad of challenging mathematical questions that have risen from the attempts to rigorously explain the power of neural networks. A prominent recent approach towards explaining the success of deep-learning has been to try to get provable linear-time training of various kinds of neural networks when their width is a high degree polynomial in the training set size, inverse accuracy and/or inverse confidence parameters (which is a somewhat {\it unrealistic} regime) \cite{lee2017deep,wu2019global,du2018gradient,su2019learning,huang2019dynamics,allen2019convergenceDNN,allen2019learning,allen2019convergenceRNN,zou2018stochastic,zou2019improved,arora2019exact,arora2019fine}. The essential proximity of this regime to kernel methods has been elucidated separately in works like  \cite{allen2019can,wei2019regularization}. Some important progress has also happened with regards to provable training of certain special finite-sized neural networks, \cite{goel2018learning,frei2020agnostic,goel2019time,diakonikolas2020approximation,mukherjee2020study,KARMAKAR2022264,KARMAKAR202356}. And most recently in \cite{gopalani2022global}, some of the current authors have shown a first-of-its-kind convergence of S.G.D. on shallow nets at arbitrary data and width. In summary, it is fair to say that the proof of convergence of neural training remains open for almost all practically relevant finitely large networks. 

Hence there has risen a recognition of an urgent need for a {\em phenomenological} understanding of the process of deep-learning - which can take the form of simple functions that can be computed at any point of the training process to reveal how the current updates on the complex model are influencing the predictions at arbitrary data. Thus such probes can be an online diagnostic for training. 

One such proposal is that of ``local elasticity'' that came to light in recent works, \cite{he2019local} and \cite{chen2020label}. It reveals an intriguing new feature of neural training, that it is ``local". The core phenomenon they highlighted can be framed as follows, that when a neural network is being trained to classify between say, \text{cat} and \text{dog} images, then the algorithm's weight update based on a sampled cat image changes the predictions for other cat images (\textit{intra-class}) much more significantly than dog images (\textit{inter-class}) and vice-versa. 

While the community lacks a complete mathematical framework to capture this phenomenon, as an attempt towards that, the above papers \cite{he2019local} \cite{chen2020label} put forward the following definition of a quantity called $S_{\rm rel}$ - which intuitively measures the fractional change in predictions caused by a S.G.D. update at a data point $\x'$ as opposed to at $\x$, when $\x$ is the sampled data point for computing the stochastic gradient. 

\begin{definition}[{\bf $S_{\rm rel}$}]\label{def:srel}
Corresponding to a real valued hypothesis class ${\cal F} = \{ f \mid f : {\cal X} \times \R^p \rightarrow \R \}$, loss function $\ell : {\cal F} \times {\cal X} \times {\cal Y} \rightarrow \R^{\geq 0}$, given any instantiation of the (possibly stochastic) training algorithm $\{ \w_t \in \R^p \mid t =1,\ldots \}$ which had access to labeled data $\in {\cal X} \times {\cal Y}$ and a learning rate $\eta$, we define  associated to it a family of non-negative reals valued stochastic processes, indexed by pairs of labeled input data $\big ( (\x,\y),(\x',\y) \big )$,

\begin{align}
    \R^{\geq 0} \ni t \mapsto \left [ S_{\rm rel}( t) \right ]_{\x', \x,\y} &\coloneqq  \frac{\abs{f(\x',\w_t^+(\x)) - f(\x',\w_t)}}{\abs{f(\x,\w_t^+(\x)) - f(\x,\w_t)}} \in \R^{\geq 0} 
\end{align} 
\[ \text{where,} \]
\[ \w_t^+(\x) = \w_t - \eta \nabla_{\w_t} \ell (\w_t, (\x,\y)) \] 
\end{definition} 

\vspace{\parskip}

{\remark  For a fixed choice of $\big ( (\x,\y),(\x',\y) \big )$ the above definition will always be understood to be applicable at only those times $t$ when $f(\x,\w_t^+(\x)) \neq f(\x,\w_t)$ and all properties of it hypothesized or proven or observed will be understood under this condition. Secondly, we shall often denote $\left [ S_{\rm rel}( t) \right ]_{\x', \x,\y}$ as  $S_{\rm rel}(\x',\x)$ or $S_{\rm rel}(t,\x',\x)$ depending on which variable dependencies we want to emphasize.}

{When $\srel$ is large for many $\x'$s spread across the data space as a response to the update $\w_t$ to $\w^+_t(\x)$, the predictions are changing more globally across the data space and we characterize the training as being ``inelastic". Alternatively, when $\srel$ is low, the training is of a more ``local" or ``elastic" nature.}

A primary point of motivation for our work is to note that in the previous studies the intricate structure in the time dynamics of the above quantity was not sufficiently studied. Also, the predictors that are trained for multi-class classification tasks have multi-dimensional outputs, while the authors in \cite{he2019local} had not proposed any extension of the definition of  $S_{\rm rel}$ for this scenario.  We launch a detailed investigation into this lacuna in existing literature - and show evidence that obvious extensions of Definition \ref{def:srel} to multi-dimensional predictors do not effectively capture the locality of the dynamics. This in turn leads us to propose a new definition of $S_{\rm rel}$ which we demonstrate to be far more sensitive for neural training dynamics while learning classification tasks. On the other hand, we also rigorously compute the original/above definition for various well-motivated regression settings and demonstrate a close match between intuition, theory and experiments in that regime. 

\subsection{An Overview}
We summarize our contributions as follows, 

\begin{itemize} 
\item  In Section \ref{sec:srelkl} we motivate and introduce a new definition of $S_{\rm rel}$ which is tuned to the setup of training a predictor which maps inputs to a probability distribution over a set of finite classes. 

\item In Section \ref{sec:conj} we outline our experimental results with both the variants of $S_{\rm rel}$ that we consider. In here we shall summarize into a set of hypotheses what our experiments reveal about the behaviour of $S_{\rm rel}$ as a function of time/iterations during the progress of a neural training algorithm.

\item In Section \ref{sec:thm_summary} we outline our theoretical results whereby we isolate certain situations of learning by gradient flow along which the time-evolution of $S_{\rm rel}$ (Definition \ref{def:srel}) can be exactly calculated. The structure of these closed-form expressions will give evidence towards the hypotheses stated in Section \ref{sec:conj}. 

\item In Section \ref{sec:svhn} and Appendix \ref{app:exp}  we give a comparative study of how our proposed definitions of local elasticity (Definition \ref{def:srel_ce_k}) is better at uncovering a phenomenon of locality in training than (the required natural extensions of) the earlier proposal in Definition \ref{def:srel} to multi-dimensional predictors. We demonstrate this via tracking both the definitions in tandem along the iterations of training practical neural nets on standard classification data. 

On the other hand, in Section \ref{exp:reg}, we give experimental studies in the regression setting and point out the ease of interpretation of $\srel$ as defined in Definition \ref{def:srel} as a measure of locality. 


\item In Section \ref{sec:disc_features}, Section  \ref{sec:relu_ode} and Section \ref{sec:d_hom} we give the formal theorem statements and the proofs of the results outlined in Section \ref{sec:thm_summary}. Various lemmas necessary in the above sections can be found in Appendices \ref{app:relu} and \ref{proof:d_hom}.
\end{itemize} 

We note two salient points about all the experiments we demonstrate -- (a) in demonstrations of the dynamics of $S_{\rm rel}$ (Definition \ref{def:srel}) in the classifications settings, we have smoothened it by a similar averaging over pairs of data as has been stated explicitly in Definition \ref{def:srel_ce_k} for the new proposal. (b) In all experiments where $\srel$ measurements have been done in Sections \ref{sec:svhn}, \ref{sec:disc_features} and Appendix \ref{app:exp}, all the plots account for the stochasticity of training and indicate the standard deviation over multiple runs of the experiment. All the codes for the experiments can be found at, \href{https://github.com/avirupdas55/Dynamics-of-Local-Elasticity-During-Training-of-Neural-Nets}{https://github.com/avirupdas55/Dynamics-of-Local-Elasticity-During-Training-of-Neural-Nets}.




\subsection{A New Definition of Local Elasticity For Simplex Valued Predictors.}\label{sec:srelkl} 

\begin{definition}[{\bf Defining $\srel$ for probability simplex valued predictors}]\label{def:srel_ce} 
Suppose $\g_\w : \R^n \rightarrow \R^\C$ is a probability simplex valued predictor i.e $\g_\w(\x)_i \in [0,1], \sum_{i=1}^\C \g_\w(\x)_i = 1 ~\forall \x \in \R^n, \forall i \in \{1,\ldots, \C \}$

\begin{align*}\label{def:norm_srel} 
    \R^{\geq 0} \ni t \mapsto S_{\rm rel}(t,\x',\x) &\coloneqq \left [ S_{\rm rel}( t) \right ]_{\x', \x,\y}\\
    &\coloneqq  \frac{ \KL \Big  (\g(\x',\w_t^+(\x)), \g(\x',\w_t)  \Big )}{ \KL  \Big  (\g(\x,\w_t^+(\x)) , \g(\x,\w_t)  \Big  )} \in \R^{\geq 0} 
\end{align*} 
\[ \text{where } \] 
\[\w_t^+(\x) = \w_t - \eta \cdot  \nabla_{\w_t} \ell (\w_t, (\x,\y)) \]
\end{definition}

Note that in the definition above the output of $\g$ is always a probability vector and we recall that given two probability vectors $\p, \q \in [0,1]^\C$ s.t $\sum_{i=1}^\C p_i = 1 = \sum_{i=1}^\C q_i$, we have, $\KL (\p,\q) = \sum_{i=1}^\C p_i \log \left ( \frac {p_i}{q_i}  \right )$

A natural use-case of the above definition is when studying the training of a network $\net$ (parameterized by weight vector $\w$) composed with a layer of soft-max, i.e., the mapping $\g_{\w} \coloneqq {\rm Soft{-}Max} \circ \net : \R^n \rightarrow \R^\C$. This composed function $\g_{\w}$ would commonly be trained via the cross-entropy loss $\ell = \ell_{\rm CE}$ on a $\C$ class labeled data. 

Further, we would like to note that  Definition \ref{def:srel} focused on the changes in the predictor's output. In contrast, our proposed Definition \ref{def:srel_ce}, tuned to the classification setup, is trying to capture the ratio between the change in the predicted distribution over the classes at the test point $\x'$ and the corresponding change in the predicted distribution over the classes at the sampled point $\x$.  

We posit that its more useful to make the new $\srel$ more sensitive to the class identities rather than to the specific members of the class being sampled in a given update. Towards that end, we define the following smoothed version of $\srel$ adapted to the classification setting,  

\begin{definition}[{\bf Defining smoothed $\srel$ for probability simplex valued predictors}]\label{def:srel_ce_k} 
Given the setup as in Definition \ref{def:srel_ce}, we further define the following smoothed quantity corresponding to any choice of $2k$ data points $\{ \x_{{\rm c}_1,i} \mid i=1,\ldots,k  \}$ and $\{ \x_{{\rm c}_2,i} \mid i=1,\ldots,k  \}$ from any two classes ${\rm c}_1$ and ${\rm c}_2$ respectively. 

\begin{equation}
\begin{split} 
&S_{\rm rel}^{\rm k, smooth}(t,\{ \x_{{\rm c}_1,i} , \x_{{\rm c}_2,i} \mid i=1,\ldots,k \})\\
&\coloneqq \frac{\sum^{k}_{i=1}\sum^{k}_{j=1}S_{\rm rel}(t,x_{{\rm c}_1,i},x_{{\rm c}_2,j})}{k^2}
\end{split}
\end{equation}
\end{definition} 

In the later sections, we will use the above smoothed version of Definition \ref{def:srel_ce} when investigating behavior of neural networks for the classification experiments.



\subsection{A Summary of Our Experimental Results About $S_{\rm rel}$ (Definitions \ref{def:srel} and \ref{def:srel_ce_k})}\label{sec:conj} 


In Section \ref{sec:svhn} we give comparative experimental studies between Definition \ref{def:srel} and  \ref{def:srel_ce_k} to demonstrate that the latter has a sharper interpretation as measuring the ``locality of updates'' in classification settings. On the other hand, in Section \ref{exp:reg}, we give experimental studies in the regression setting and point out the ease of interpretation of $\srel$ as defined in Definition \ref{def:srel} as a measure of locality. These experiments can be seen to be leading to the following hypothesis about the dynamics of neural training, 



\begin{hypothesis}[{\bf Proximity Conditions for Smallness of $S_{\rm rel}$ of Definition \ref{def:srel}}]\label{conj:small}
There exists learning situations as described in the setup of Definition \ref{def:srel} with ${\cal X} = \R^n$ such that,
\begin{enumerate}[(a)] 
\item $\exists$ a metric $g$ on the data space s.t for any given $\x'$ and $\x_h$, $\exists$ $\x_\ell$ s.t $g(\x',\x_\ell) < g(\x',\x_h)$ and $S_{\rm rel}(\x',\x_{\ell}) > S_{\rm rel}(\x',\x_h)$. 
\item There exists a function $\funcP : \R^n \times \R^n \rightarrow [0,\infty)$ which is symmetric in its inputs  s.t, 
$S_{\rm rel} (\x,\x') = {\mathcal O} \left (  \funcP (\x,\x') \right )$ 
\end{enumerate} 
\end{hypothesis}



Firstly we note that the neither of the two parts of the hypothesis above depend on time/stage of training. Secondly, note that the part (a) above allows for the fact that the graph of $[S_{\rm rel}(t_*)]_{\x',\x,\y}$ versus $g(\x',\x)$, for a fixed $\x'$, could have a local maxima at some $\x$. We rather posit that this plot would have a downward trend in general -- that for a fixed $\x'$ and $\x_h$, we can always find a data $\x_\ell$ which is nearer to $\x'$ and has a higher $S_{\rm rel}$. 

In other words, the influence of the prediction at $\x'$ of an update based on sampling $\x$, diminishes without too many abrupt changes as $\x$ gets farther away from $\x'$ in the metric $g$. Note that this $g$ could be very different from the natural metric in the space where the data is presented during training - but it is rather sensitive to the features of it as extracted by the learning process at hand. 

Further, the function $\funcP$ hypothesized to exist in (b) above gives a {\it time-independent} proximity/``relatedness" condition on the pair $(\x,\x')$ which if low, would be sufficient for $S_{\rm rel}(\x,\x')$ to be low - and this $\funcP$ is a priori expected to be different than $g$ from the first part.  Hence, at all times $[S_{\rm rel}(t)]_{\x',\x,\y}$ is larger for those $\x$ and $\x'$ which are closer to each other. {\it Thus the influence of the sampled data maintains the local nature of its influence at all times.}

\begin{hypothesis}[{\bf A $2-$phase Behaviour in Time is Seen by $S_{\rm rel}$ of Definition \ref{def:srel}}]\label{conj:large}
We continue in the same setup as in Hypothesis \ref{conj:small}.
\begin{enumerate}[(a)]
\item {\bf [Early times]}  $\exists$  time instants $0 \leq t_1 < t_2 < t_* < \infty$  s.t $\forall (\x,\x') ~S_{\rm rel}(\x,\x')$ in the time interval $[t_1,t_2]$ is smaller than any value it attains for $t \geq t_*$.
\item {\bf [Late times]} $\forall (\x,\x'), ~\forall t \in [t_*,\infty)$, $S_{\rm rel} (\x,\x') > C(t,\x,\x') >0$ for a function $C : \R \times \R^n \times \R^n \rightarrow (0,\infty)$ which is non-decreasing in its $t-$dependence.
\end{enumerate}
\end{hypothesis} 

We recall that in above, the $t-$dependence of $S_{\rm rel}$ has been suppressed for notational clarity. Also, just as the late time lower bound $C(t,\x,\x')$ is expected to have a non-trivial dependence on $(\x,\x')$, the range of values attained by $S_{\rm rel}(\x,\x')$ in the early times, $[t_1,t_2]$, is also expected to have a strong dependence on the data pair $(\x,\x')$. We anticipate that the there exists proximity conditions on the pair $(\x, \x')$ which if true/false then it would consequentially raise/lower the value of $S_{\rm rel}(\x,\x')$ for $t \in [t_1,t_2]$ and also the value of $C(t,\x,\x')$ for all $t > t_*$ -- while still maintaining the distinction between the early time (before $t_*$) and late time (after $t_*$) behaviour. Hence in words, {\it we observe that neural regression is characterized by a time scale $t_*$ about which they undergo a ``phase transition" - that compared to the early times, they behave ``less elastically" in the later stages of the process.}

Our experiments also reveal another distinguishing feature between the two phases above -- that during the initial elastic phase, $S_{\rm rel}(t)$ is a rapidly increasing function of time while in the late phase it is not.





Finally, we note the consistent behaviour that is observed for $S_{\rm rel}(t)^{k,{\rm smooth}}$ as given in Definition \ref{def:srel_ce_k}. Recalling that ${\rm c}_2$ is the class from which the fictitious S.G.D. update sampled the data, we observe that whenever ${\rm c}_1 = {\rm c}_2$ i.e $S_{\rm rel}(t)^{k,{\rm smooth}}$ is being measured intra-class, it turns out to be fairly time independent and large than when ${\rm c}_1 \neq {\rm c}_2$. Thus $S_{\rm rel}(t)^{k,{\rm smooth}}$ quantifies how local to the true class of the sampled data are the changes that the algorithms make to the predictor while doing classification. 

We posit that being able to rigorously prove the above properties of $\srel$ would be a major advance towards demystifying deep-learning.

\subsection{A Summary of Our Theoretical Results About $S_{\rm rel}$ (Definition \ref{def:srel})}\label{sec:thm_summary} 

We isolate three models of learning by integrable gradient flows and along their solutions we give closed-form expressions for the time evolution of $\srel$ as given in Definition \ref{def:srel}. When the structure of these closed-form expressions are studied, they are seen to give evidence for many of the hypothesized behaviour of $\srel$ on neural networks. 

In Section \ref{sec:disc_features} we consider a model of training the last layer of a neural network. Here we show that the gradient flow on a penalized regression over features with {\it discrete labels} is exactly integrable. We use this to derive a closed-form expression for $\srel$ for this model and plot its dynamics for synthetic classification data. 

Next, we move on to consider $\srel$ for two exactly solvable gradient flow models of regression with {\it continuous labels}. In Section \ref{sec:relu_ode} we analyze learning of a $\relu$ gate, a situation where the predictor function is not a polynomial in the weights. In Section \ref{sec:d_hom} we analyze a situation where the predictor functions being trained are polynomial in their parameters/``weights".

In the following table we summarize the theoretical results presented in this paper. In light of that we note the following salient points. {\rm (A)} {In each of the above cases we get a closed form expression for $\srel$ {\it without having to make any specific distributional choice for the data except for realizability.}} {\rm (B)} The plots of the dynamics of $\srel$ immediately reveal a two-phase behaviour of training -- an initial elastic phase when $\srel$ changes rapidly and an eventual inelastic phase when $\srel$ remains larger than during initial times. 

{\rm (C)} Lastly, we emphasize that for the above models we also demonstrate that the theoretically derived behaviour for $S_{\rm rel}$ along the chosen gradient flows closely matches the time evolution of the $S_{\rm rel}$ along the S.G.D. trajectory - whose parameters were chosen to mimic the gradient flow. Thus along the way we give evidence that there is a well-defined continuous time limit of the dynamics of $S_{\rm rel}$ along S.G.D. updates. 

\begin{center}
    \begin{tabular}{|M{0.27\linewidth} | M{0.16\linewidth} |  M{0.22\linewidth}| M{0.29\linewidth}|}
    \hline 
    {\bf Model} & {\bf Result} & {\bf Hypothesis That it Corroborates } & {\bf Example of Exact $S_{\rm rel}$ \newline} \\
    \hline
    \hline
    Classification by features & Theorem \ref{thm:srel_peel} & Hypothesis \ref{conj:large} (b) & Theorem \ref{thm:srel_peel} 
\\
    \hline 
     A single  $\relu$ gate & Lemma \ref{lem:low_srel_relu} & Hypothesis \ref{conj:large} (b) &
(implicitly contained in the discussion)\\
    \hline 
    {Weight  $d-$homogeneous, \newline feature linear predictors}  & Theorem  \ref{thm:srel_cont_upper} \newline Theorem  \ref{thm:2_2_diag}  & Hypothesis \ref{conj:small} (b) \newline Hypothesis \ref{conj:large} (b)  &Theorem \ref{thm:exact_S_rel}\\
     \hline 
    \end{tabular} 
\end{center}

{\em The plots of these exact $S_{\rm rel}$ formulas} -- in the table above - {\em provide empirical evidence for the first part of both the above hypothesis.}

\subsection{A Review of Related Literature}

Closely related to the idea of elasticity, the authors in \cite{stiffness} postulated the concept of ``stiffness'' which tries to capture the similarity of updates between two data by measuring a notion of similarity between the gradients at those two data points. The following two definitions were used by them to measure inter-class and intra-class stiffness respectively, 
$S_{\rm sign}\left(\left({\mathbf x}_1, y_1\right), \left({\mathbf x}_2, y_2\right);f_W\right)= \mathbb{E}\left[{\rm sign} \left( {\mathbf g}_1^\top {\mathbf g}_2\right)\right]$ and $S_{\rm cos}\left( \left ({\mathbf x}_1, y_1\right), \left({\mathbf x}_2, y_2\right);f_W \right)= \mathbb{E}\left[{\rm cos} \left( {\mathbf g}_1,{\mathbf g}_2\right)\right]$
where ${\mathbf g}_1=\nabla_W\mathcal{L}\left(f_W\left({\mathbf x}_1\right),y_1\right)$ , ${\mathbf g}_2=\nabla_W\mathcal{L}\left(f_W\left({\mathbf x}_2\right),y_2\right)$ and $W$ denotes the trainable parameters and $f_W$ is the loss at the given weight values. These measures were used to show how the stiffness between two data points depends on their mutual distance in the input space, and thus establish the concept of a dynamical critical length 
 i.e a distance threshold s.t another point has to be within that distance from the sampled data so that updates on the parameters based on the sampled data influences the former. Firstly, unlike our formulations of a uniform notion of influence, this proposal needed to switch between two different notions of stiffness based on whether the two points are in the same class or not. We posit that our function's graphs are much easier to interpret because of this functional uniformity across all pairs of data. Secondly, there is lack of evidence if stiffness is meaningful in regression settings while our proposal has a regression analogue s.t we can demonstrate non-trivial learning scenarios where we can analytically compute our measure and match it to experimental values. To the best of our knowledge, there is no analytical form known for the time-dynamics of the stiffness metric in any setting. 

\cite{coherent_grads} put forth the coherent gradient hypothesis which posits that the gradients of the loss function at similar data tend to aligned to each other and those mini-batch updates cause the maximum change in loss when per-example gradients are more aligned with each other. They further hypothesized that these mini-batch updates would cause larger drops in the value of the loss function than when alignment is lower - and such updates cause beneficial changes to prediction at a larger number of samples. This can be seen as one of the first works which postulated a relationship between similarity of data in a mini-batch and that update pushing the network towards good generalization. One can see our direction of investigation with local elasticity as being a mathematically concrete way to instantiate these ideas. 

Another important elasticity based advancement towards understanding generalisation was made by \cite{locally_elastic_stability} via defining a property called ``local elastic stability'', where an algorithm $\mathscr{A}$ is said to have local elastic stability $\beta_m(\cdot, \cdot)$ with respect to the loss function $l$ if, $\forall \ m$ the following inequality holds $\forall\ S\in \mathcal{Z}^m, \ 1\leq i\leq m, $ and $z\in \mathcal{Z}$, $\left| l(\mathcal{A}_S,z)-l(\mathcal{A}_{S^{-i}},z) \right| \leq  \beta_m(z_i, z)$.
In above, $\mathcal{A}_S$ is the output of the learning algorithm $\mathcal{A}$ for the $m-$sized input training set $S$ whose each data is sampled i.i.d from the probability space $\mathcal{Z}$ and $\mathcal{A}_{S^{-i}}$ is the output without the $i^{\rm th}-$ data of this training set. However, we note that all analysis done therein for the property of local elastic stability is contingent on the assumption that the loss function is convex in nature. 

In contrast to the above, in this work, we are able to establish analytic insights (which bear out in experiments) for the notion of local elasticity for multiple well-motivated non-convex losses.

\section{Experimental Study of the Time Dynamics of  $\srel$}\label{sec:exp}

We will first present a thorough empirical study of the time dynamics of $S_{\rm rel}$ using our newly proposed Definition \ref{def:srel_ce_k} for the benchmark classification tasks on the datasets: SVHN \cite{netzer2011reading} \footnote{SVHN is a real-world image classification dataset that is similar in flavor to MNIST but incorporates more labeled data which comes from a significantly harder, unsolved, real world problem of recognizing digits and numbers in natural scene images. SVHN is obtained from house numbers in Google Street View images and it consists of $73257$ digits for training and $26032$ digits for testing.}, CIFAR-10 and CIFAR-100 \cite{krizhevsky2009learning}\footnote{CIFAR-10 consists of 10 mutually exclusive classes containing different animals and vehicles. It has $50,000$ training images and $10,000$ test images with $6,000$ images in each class, while CIFAR-100 has $600$ images in each class.}

In this section, the context shall always be of neural nets being trained to classify via minimizing a cross-entropy loss. We note that when Definition \ref{def:srel} is used in this setup, there is flexibility about whether the predictor $f$ is to be taken as the output of the neural network or the output composed with a softmax layer. We conducted our experiments for both options and found the latter performs better. Secondly, implicitly in Definition \ref{def:srel} there is also the flexibility of choosing the norms to be used in the image space of the predictor for evaluating the required ratio. We conducted the experiments using both $\ell_1$ and $\ell_2$-norms, but the results were essentially indistinguishable.

So we report our experimental comparisons between Definitions  \ref{def:srel} and \ref{def:srel_ce_k} only for the case of evaluating the former with $f ={\rm Soft{-}Max}\circ {\rm Net}$ and using $1-$norm on the probability simplex. In the following section we report our studies with SVHN and CIFAR-10 and the corresponding results on CIFAR-100 are presented in Figure \ref{fig:srel_cifar100} in Appendix \ref{app:exp}.

\subsection{A Comparative Study of $\srel$ Definitions \ref{def:srel} and \ref{def:srel_ce_k}}\label{sec:svhn}
\label{exp:svhn}

For this experiment we train a ResNet-18 architecture, composed with a soft-max layer using the cross-entropy loss, and identify the top-$3$ classes with the highest class-wise accuracy. We then train another instance of ResNet-18
model \footnote{For the experiments in this section we remove batch-normalization from the ResNet-18 model, so as to facilitate defining gradient updates with mini-batch size $1$ which is necessary to define $\w_t^+$ in Definitions \ref{def:srel} and Definition \ref{def:srel_ce_k}.}, again with softmax and the cross-entropy loss, on these top-$3$ classes. We choose $k=20$ training data from each class and track the smoothed version of $S_{\rm rel}$ i.e $S_{\rm rel}^{\rm k, smooth}$ as given in Definition \ref{def:srel_ce_k} for all the $9$ possible class pairs. For a head-to-head comparison we also track for each of those classes the original definition of $S_{\rm rel}$ as given in Definition \ref{def:srel} (similarly smoothened by averaging over data pairs as in Definition \ref{def:srel_ce_k}). Some further details of the experimental setup are as follows in the table,


\begin{table}[tbh]
  \begin{center}
    \begin{tabular}{|l|r|}
   \hline     
      Parameter & Value  \\
   \hline
   Depth of  ResNet &  18 (no batch norm)\\
   The map that the $\net$ implements & $\R^{3 \times 32 \times 32} \to \R^{10}$\\ 
   \hline 
      Optimizer & ADAM \\
      Learning rate & $1\cdot 10^{-4}$\\
      Mini-batch size & 250\\
      Dropout or any other regularization & Not Used\\
      \hline
  \end{tabular}
  \end{center}
  \label{table_svhn_no_bn}
\end{table}

On the left column of Figure \ref{fig:srel_svhn}, we see the time evolution of $\srel^{\rm k, smooth}$ for all the $9$ possible class pairs corresponding to images from the top 3 classes, $\lbrace 0,2,4\rbrace$ for SVHN. And the in the right column of Figure \ref{fig:srel_svhn} we see the same for the original definition of $\srel$. In Figure \ref{fig:srel_cifar_plane_car_ship} a similar tracking for both the definitions is done for the top 3 classes for CIFAR-10 as detected by ResNet-18, namely $\lbrace plane, car, ship \rbrace$. 

\begin{figure}
    \centering 
\begin{subfigure}{0.50\textwidth}
  \includegraphics[width=1\textwidth]{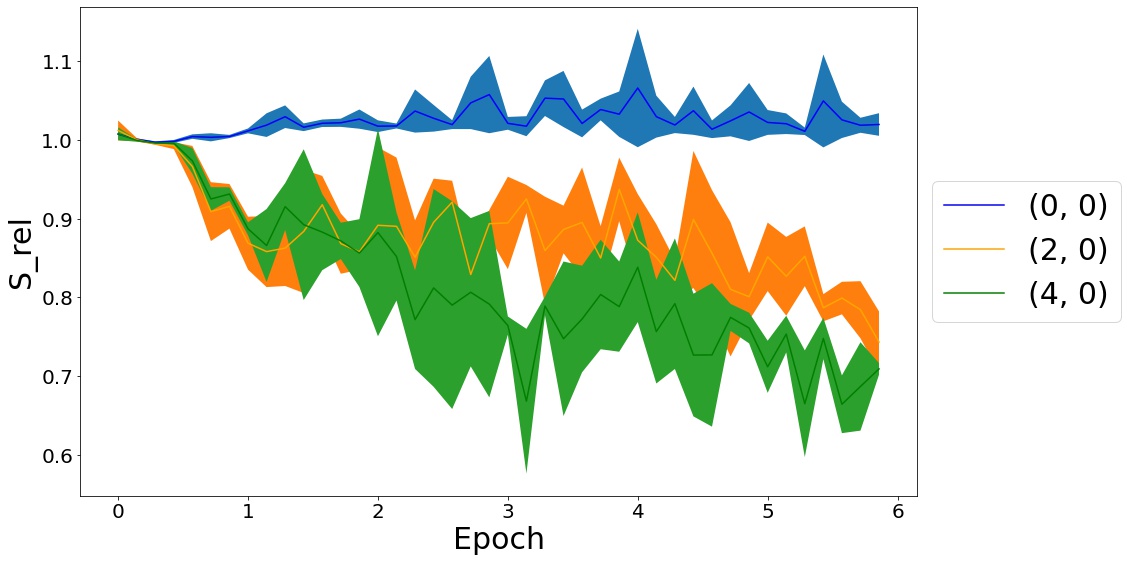}
  \caption{}
  \label{fig:1a}
\end{subfigure}\hfil 
\begin{subfigure}{0.50\textwidth}
  \includegraphics[width=1\textwidth]{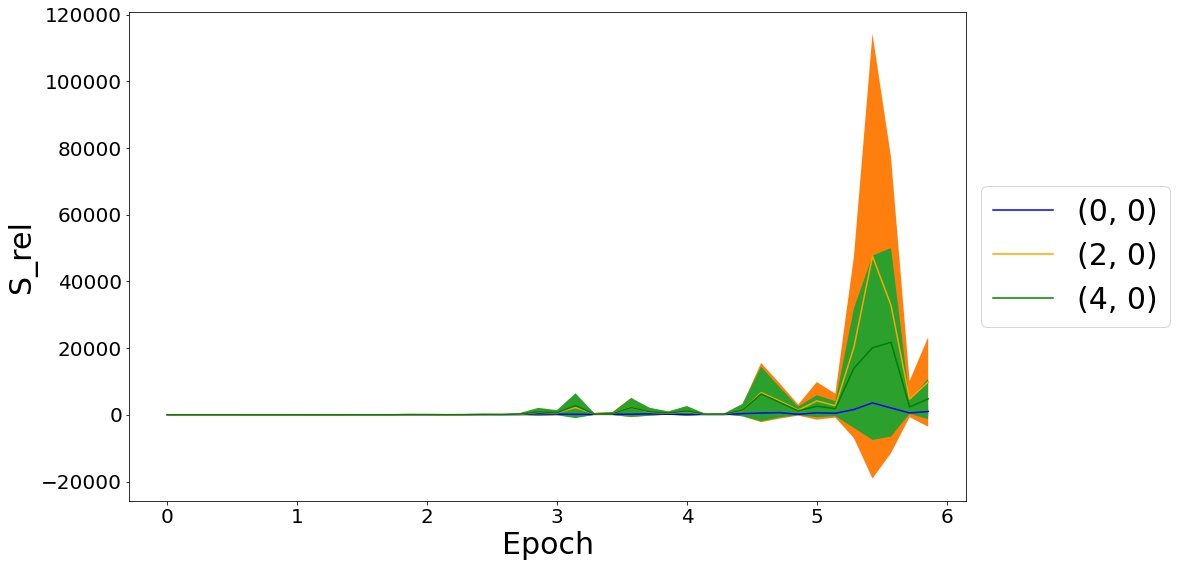}
  \caption{}
  \label{fig:1b}
\end{subfigure}
\medskip
\begin{subfigure}{0.50\textwidth}
  \includegraphics[width=\linewidth]{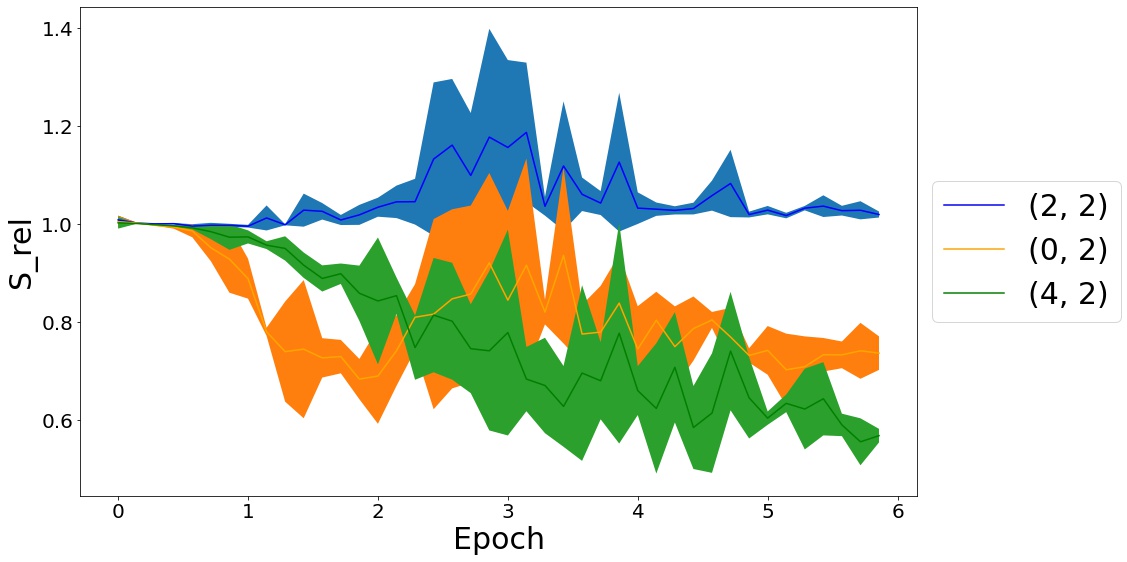}
  \caption{}
  \label{fig:1c}
\end{subfigure}\hfil 
\begin{subfigure}{0.50\textwidth}
  \includegraphics[width=\linewidth]{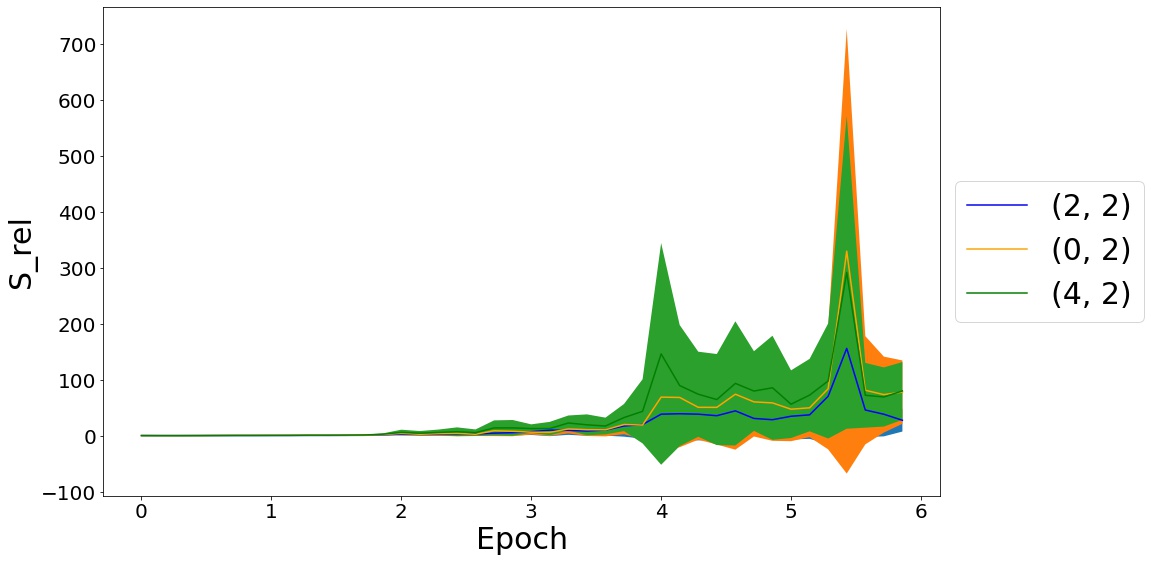}
  \caption{}
  \label{fig:1d}
\end{subfigure}
\medskip
\begin{subfigure}{0.50\textwidth}
  \includegraphics[width=\linewidth]{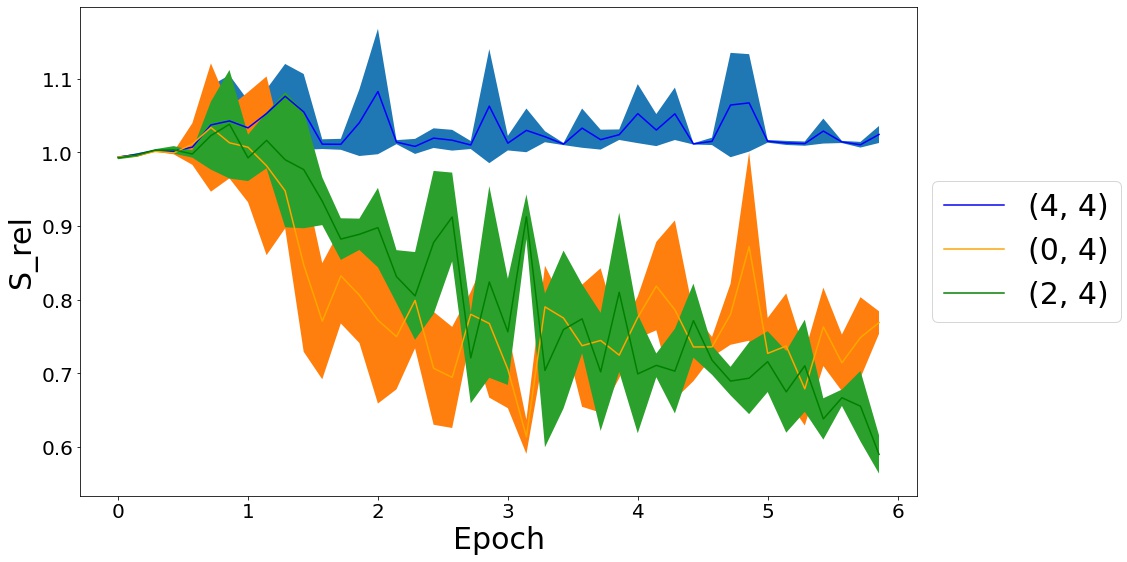}
  \caption{}
  \label{fig:1e}
\end{subfigure}\hfil 
\begin{subfigure}{0.50\textwidth}
  \includegraphics[width=\linewidth]{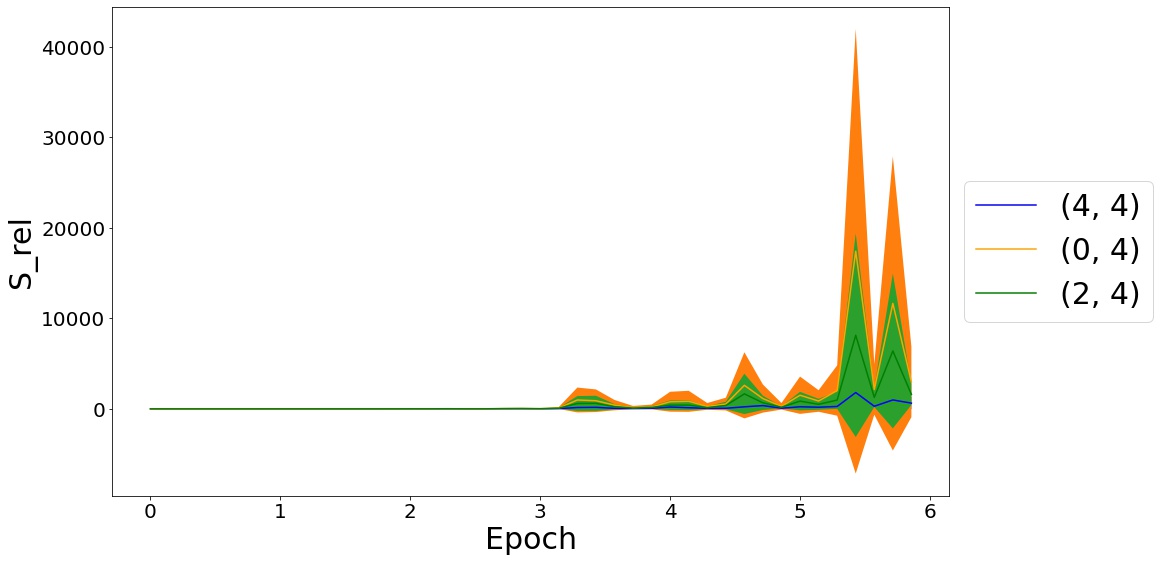}
  \caption{}
  \label{fig:1f}
\end{subfigure}
\caption{ Time evolution of $S_{\rm rel}^{\rm k, smooth}$ (Definition \ref{def:srel_ce_k}) in figures (a), (c) and (e) and $S_{\rm rel}$ (Definition \ref{def:srel}) -  smoothed by averaging over data pairs in figures (b), (d) and (f) for the two image classes being chosen from $\lbrace\text{0, 2, 4} \rbrace$ with {\rm Class-$\x$} being ${0}$ in (a) and (b), ${2}$ in (c) and (d), and ${4}$ in (e) and (f). Figure \ref{fig:loss_acc_dynamics_svhn} shows the loss and accuracy dynamics for the experiment.}
\label{fig:srel_svhn}
\end{figure}

\begin{figure}
    \centering 
\begin{subfigure}{0.50\textwidth}
  \includegraphics[width=1\textwidth]{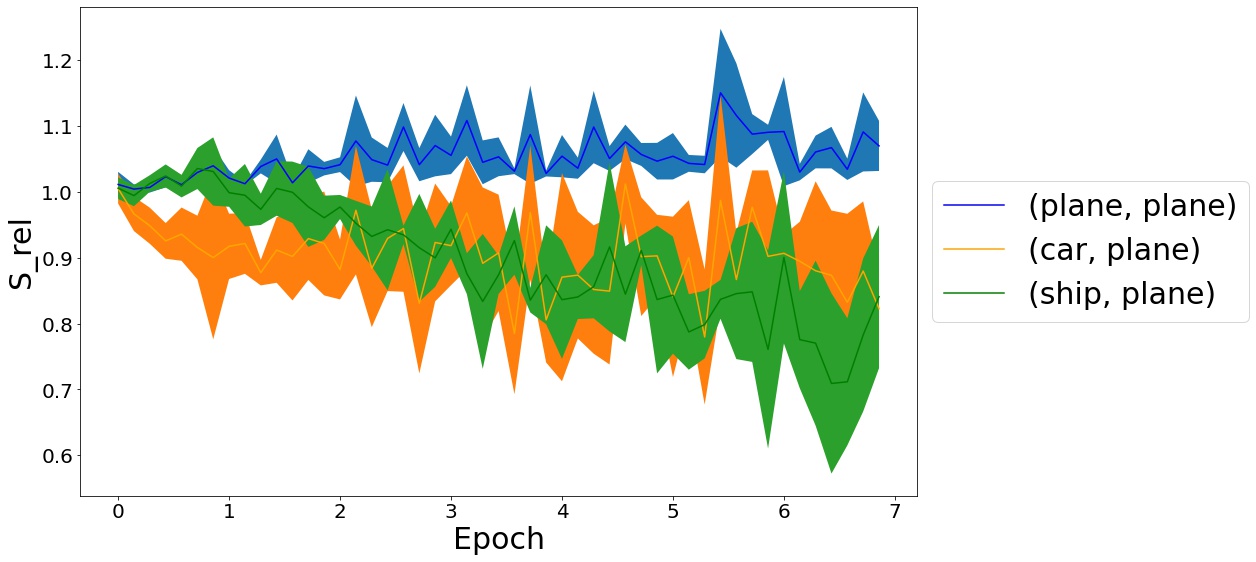}
  \caption{}
  \label{fig:2a}
\end{subfigure}\hfil 
\begin{subfigure}{0.50\textwidth}
  \includegraphics[width=1\textwidth]{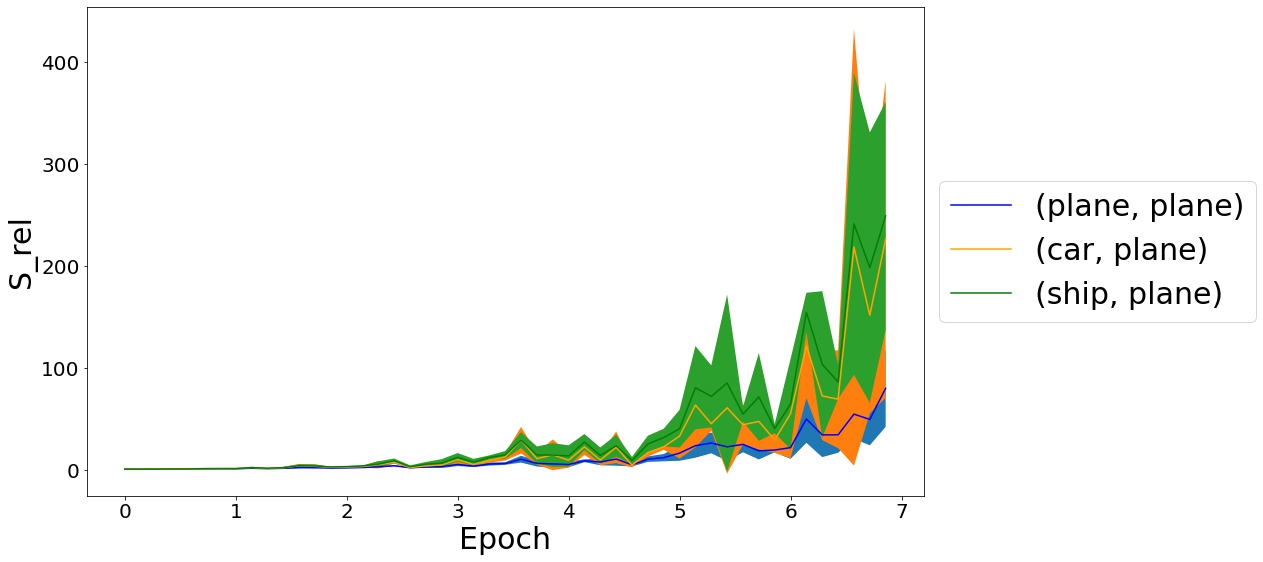}
  \caption{}
  \label{fig:2b}
\end{subfigure}
\medskip
\begin{subfigure}{0.50\textwidth}
  \includegraphics[width=\linewidth]{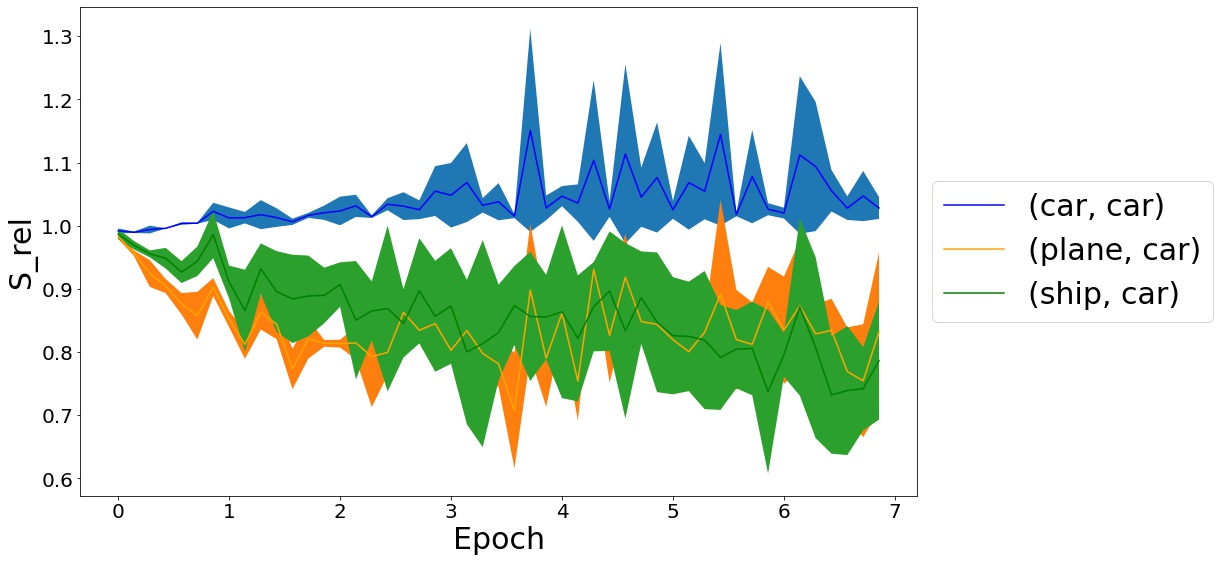}
  \caption{}
  \label{fig:2c}
\end{subfigure}\hfil 
\begin{subfigure}{0.50\textwidth}
  \includegraphics[width=\linewidth]{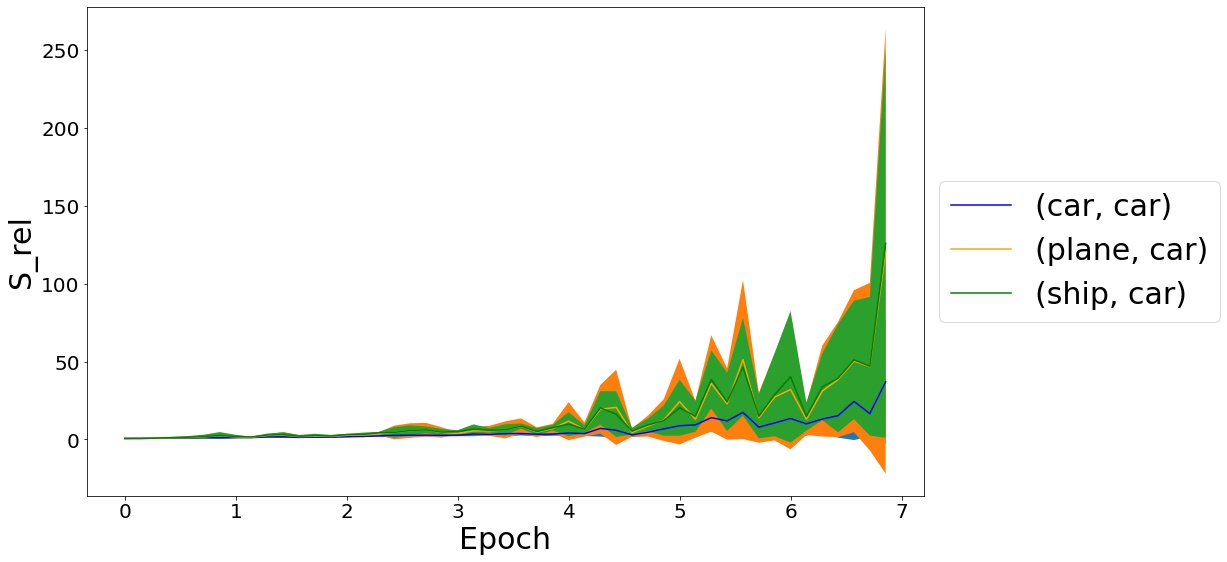}
  \caption{}
  \label{fig:2d}
\end{subfigure}
\medskip
\begin{subfigure}{0.50\textwidth}
  \includegraphics[width=\linewidth]{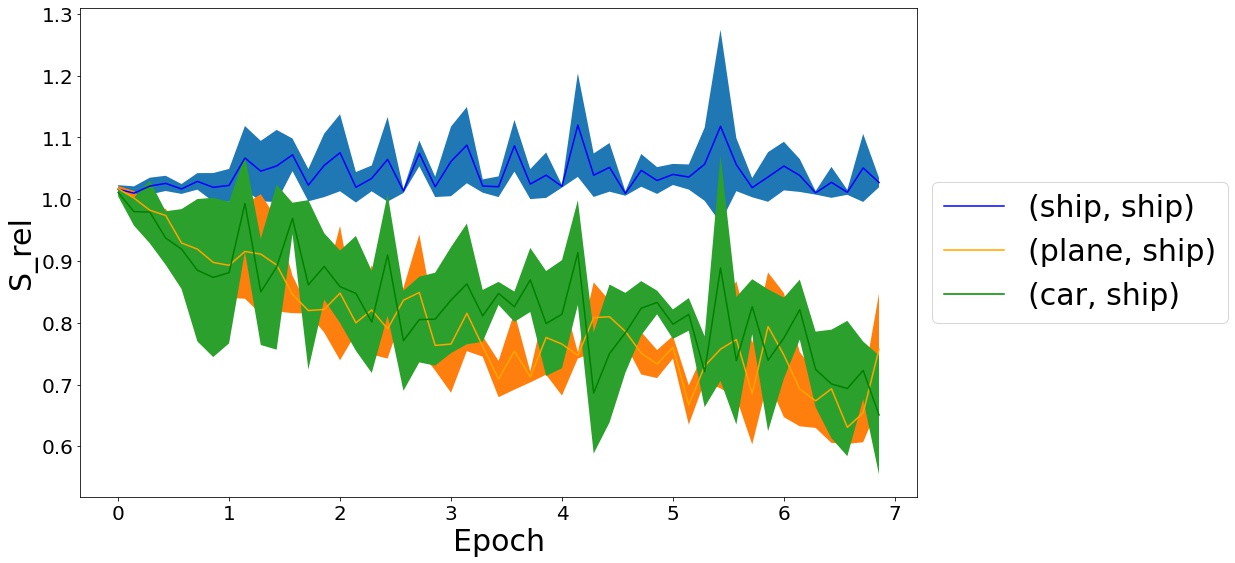}
  \caption{}
  \label{fig:2e}
\end{subfigure}\hfil 
\begin{subfigure}{0.50\textwidth}
  \includegraphics[width=\linewidth]{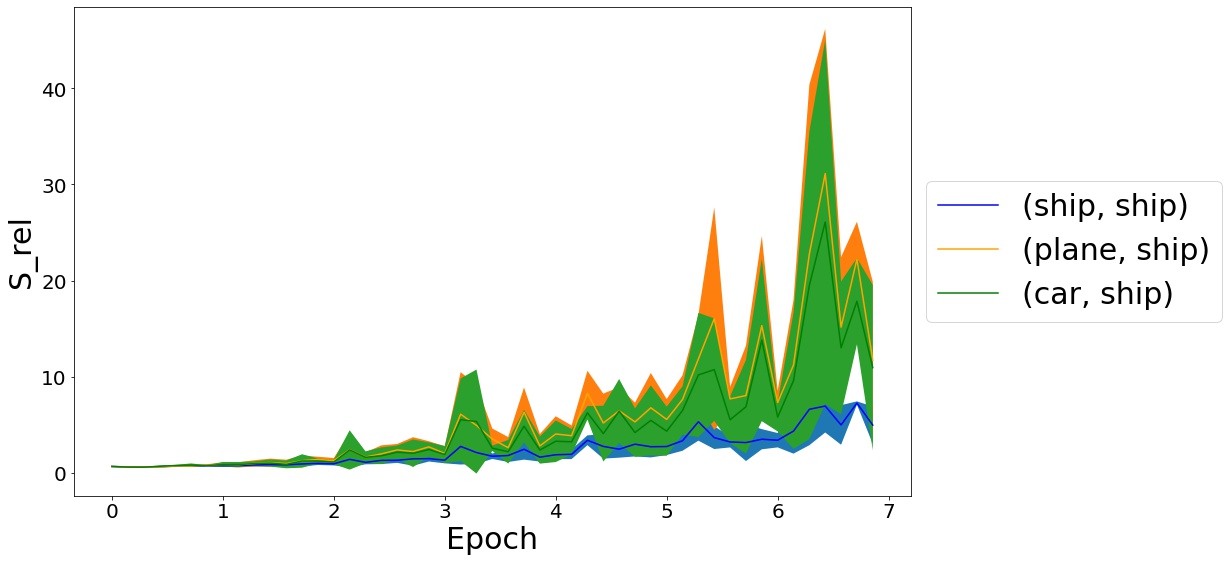}
  \caption{}
  \label{fig:2f}
\end{subfigure}
\caption{Time evolution of $S_{\rm rel}^{\rm k,smooth}$ (Definition \ref{def:srel_ce_k}) in figures (a), (c) and (e) and $S_{\rm rel}$ (Definition \ref{def:srel}) -  smoothed by averaging over data pairs in figures (b), (d) and (f) for the two image classes being chosen from $\lbrace\text{plane, car, ship} \rbrace$ with {\rm Class-$\x$} being \emph{plane} in (a) and (b), \emph{car} in (c) and (d), and \emph{ship} in (e) and (f). Figure \ref{fig:loss_acc_dynamics_cifar10} shows the loss and accuracy dynamics for the experiment.}
\label{fig:srel_cifar_plane_car_ship}
\end{figure}

For each pair of classes, all the plots show the mean $\srel$ value (for both the definitions tested)  averaged over $5$ random runs as well as the standard deviation over those runs. We observe that for all pairs, the intra-class $\srel^{\rm k, smooth}$ value as given by the current proposal is higher than the inter-class $\srel^{\rm k, smooth}$ value computed by the same. Further, while the intra-class $\srel^{\rm k, smooth}$ value remains relatively constant through the training process, the inter-class $\srel^{\rm k, smooth}$ value progressively decreases as training progresses. 

But such a separation is not seen for the original definition of $\srel$ which is tracked on the right columns of Figures \ref{fig:srel_svhn} and \ref{fig:srel_cifar_plane_car_ship}. Thus we are led to conclude that the current proposal is a more accurate view of the phenomenon of ``local elasticity'' in doing classification by deep-learning - that neural nets indeed tend to make more changes, seen in the fractional change of KL distance of the output distributions, on data in the same class as the sampled data than in other classes.

\subsection{A Study of $\srel$ (Definition \ref{def:srel}) in a Regression Setup}\label{exp:reg}

In this section, we perform an empirical study of the time evolution of $\srel$ as in Definition \ref{def:srel} for a regression problem over synthetic data with deep nets. We consider doing regression using feed-forward neural networks over an instance of non-realizable data, i.e, we construct the data $(\x,y)$  s.t  $\x$ is sampled from a standard normal distribution and $y = \norm{\x}_1$. We choose the architecture and the hyper-parameters of the experiment as tabulated below.

\begin{table}[tbh]
  \begin{center}
    \begin{tabular}{|l|r|}
   \hline     
      Parameter & Value  \\
   \hline
   Depth of the net &  5\\
   The (uniform) width of the net & 90\\
   The map that the net implements & $\R^{50} \to \R$\\ 
   \hline 
      Training set size & 16000 \\
      Test set size & 4000 \\
      Optimizer & ADAM \\
      Learning rate & $1\cdot 10^{-4}$\\
      Mini-batch size & 128\\
      Dropout or any other regularization & Not Used\\
  \hline
  \end{tabular}
  \end{center}
  \label{table}
\end{table}

\begin{itemize}
\item In Figure \ref{fig:non_real_srel_d} for a fixed $\x'$ we choose many different values of $\x_i$ and plot $S_{\rm rel}(t,\x',\x_i)$ vs $\norm{\x' - \x_i}$. {\it We do this plot at different $t=100,200,400,800,1600,2400$ and show how the trends as proposed in Hypothesis \ref{conj:small} happen at all times.}


\item In Figure \ref{fig:non_real_two_srel}  we choose two pairs of $(\x,\x')$ from the above with their $\norm{\x'-\x}$s being different. Then we plot the mean value of $S_{\rm rel}(t,\x',\x)$ vs $t$ for each of these pairs for $t = 1,..,2500$, averaged over $5$ random runs. We also show the variance of the values over these runs for each pair, via the width of the shaded areas. {\it Thus here we see that Hypothesis \ref{conj:large} bears out.}

\item In Figure \ref{fig:non_real_four_srel} we show the above plot overlaid against how the population and the empirical risks fall with time. {\it From here we realize that the initial rising part of $S_{\rm rel}(t)$ for any choice of $(\x,\x')$ coincides with the phase of training when both the risks are falling rapidly.}
\end{itemize}

We will see the key features of Fig. \ref{fig:non_real_srel_d} and \ref{fig:non_real_two_srel} be reproduced in the theoretical model to be presented in Section \ref{sec:d_hom}.



\begin{figure}[h]
\begin{minipage}[b]{.45\textwidth}
\centering
\includegraphics[height = 2in, width = 3.1in]{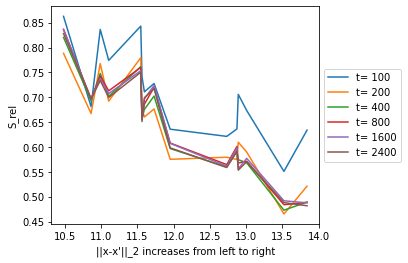}
    \caption{$S_{\rm rel}(t,\x',\x)$ vs $\norm{\x-\x'}_2$ (at a {\it fixed} $\x'$), at different times.}
    \label{fig:non_real_srel_d}
\end{minipage}
\hfill
\begin{minipage}[b]{.45\textwidth}
\centering
 \includegraphics[height = 1.6in, width = 3.1in]{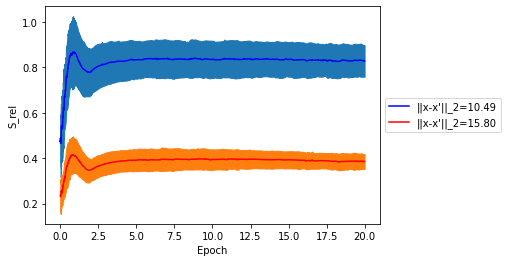}
    \caption{$S_{\rm rel}$'s time evolution at two pairs of points of $(\x,\x')$ at different mutual distances for a {\it fixed} $\x'$. }
    \label{fig:non_real_two_srel}
\end{minipage}
\end{figure}

\begin{figure}[h]
    \centering
   \includegraphics[height = 2.4in, width = 4.0in]{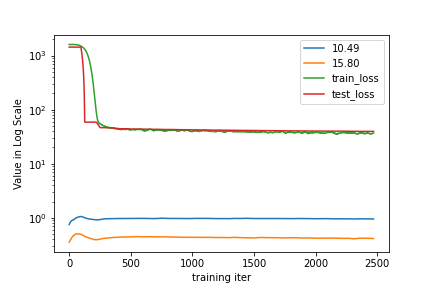}
    \caption{In here we plot the mean values of the plots of Figure \ref{fig:non_real_two_srel} overlaid against the time evolution of the mean values of the empirical and the population risks during the same experiment.}
    \label{fig:non_real_four_srel}
\end{figure}

\section{Analytic Tracking of $\srel$ for Penalized Linear Regression over Features with Discrete Labels}\label{sec:disc_features} 

We recall that $\srel$, as in Definition \ref{def:srel} was defined such that it evaluates to a (data dependent) constant for the special case of $\ell$ being the $\ell_2-$loss and $f$ being linear in the weights $\w$. But in this section we will see that regularization is sufficient to endow this $\srel$ with non-trivial dynamics - and thus we unravel a simple mechanism for $\srel$ to have the kind of non-trivial dynamics that was demonstrated in the context of neural training in the previous subsection. 

Also, penalized linear regression over features can be imagined as training over the last layer of a neural net where the fixed features are the outputs of the last layer of activations on the training data. Thus motivated we define the following key quantities corresponding to the last layer of a neural net,

\begin{itemize} 
\item Let the last layer be mapping, $\R^p \rightarrow \R^K$ with its weight matrix being $\W \in \R^{K \times p}$ and its $i^{th}-$row being $\w_i \in \R^p$. 
\item Corresponding to $K$ possible classes we consider the training data set to be of the form $(\x_i,\y_i)$ where for each $k \in \{1,\ldots,K\}$, $n_k \in \Z^+$ of these tuples are s.t their $y_i = \e_k \in \R^K$ (the $k^{th}$ standard basis in $\R^K$). Hence we can define the total size of the training data as $\cN \coloneqq \sum_{k=1}^K n_k$

\item We assume an arbitrary ordering among the $n_k$ training data in the $k^{th}-$class. Thus $\h_{k,i} \in \R^p$ would be defined as the output of the gates at the last hidden layer when the input is the $i^{th}-$training data of the $k^{th}-$class.
\end{itemize}

Hence given a positive constant, $\lambda_1$  we define the following loss function for the last layer, 
\begin{gather}\label{eq:loss_peel}
    \cL (\W) = \frac{1}{2\cdot \cN} \sum_{k=1}^{K} \sum_{i=1}^{n_k} \norm{\e_k - \W\h_{k,i}}^2 + \frac{\lambda_1}{2K} \sum_{k=1}^K \norm{\w_k}^2   
\end{gather}

We note the close resemblance between the loss function above and the models in \cite{cite-key} and \cite{DBLP:journals/corr/abs-2101-12699} where the concept of neural collapse was elucidated.

It is clear that the above loss function explicitly biases the optimization towards finding solutions where, $\frac{1}{2K} \sum_{k=1}^K \norm{\w_k}^2$ is small. We can rewrite the above in terms of loss functions evaluated at each data point, $\cL(\W ; \h_{k,i})$. Thus we have,

\begin{align*}
    \pdv{\cL}{\w_q} &= \frac{1}{\cN} \cdot \sum_{k'=1}^K \sum_{i'=1}^{n_k} \pdv{\cL(\W;\h_{k',i'})}{\w_q}\\
    &= \frac{1}{\cN} \cdot \sum_{k'=1}^K \sum_{i'=1}^{n_k} \left ( \left ( \delta_{k',q} - \ip{\w_q}{\h_{k',i'}} \right ) (-\h_{k',i'}) + \frac{\cN \lambda_1}{K} \w_q \right )\\
    &= - \left ( \frac{1}{\cN} \sum_{k'=1}^K \sum_{i'=1}^{n_k} \delta_{k',q} \h_{k',i'} \right )\\
    &+ \frac{\cN \lambda_1}{K} \w_q +  \left (  \frac{1}{\cN} \cdot \sum_{k',i'} \h_{k',i'} \h_{k',i'}^\top \right ) \w_q\\
\end{align*}

We define the following constants, 

\[ \beta^2 \coloneqq 1, \u_q \coloneqq \frac{1}{\cN} \cdot \sum_{i'=1}^{n_q} \h_{q,i'} \text{ and } \R^{p \times p} \ni \M \coloneqq   \frac{\cN \lambda_1}{K} \I_p + \frac{1}{\cN} \cdot \sum_{k',i'} \h_{k',i'} \cdot \h_{k',i'}^\top \] 

Thus we can rewrite the above gradient flow as, 

\[ \pdv{\cL}{\w_q} = -\beta^2 \u_q + \M \w_q \]

\begin{definition}\label{def:ode_w}
Given a $\theta \in \R$, we define the gradient flow time evolution of the rows of the matrix $\W$ to be happening via the following O.D.E., 
\[ \dv{\w_q}{t} = - \theta^2 \cdot \pdv{\cL}{\w_q} = \theta^2 \beta^2 \u_q - \theta^2 \M \w_q  \] 
\end{definition}

\begin{theorem}\label{thm:ode_peel}
Since $\M$ is P.D (given $\lambda_1 >0$), we can integrate the O.D.E in Definition \ref{def:ode_w}, to get, 
\[ \w_q (t) = e^{-\theta^2 \M t} \left [ \w_q(0) - \beta^2 \M^{-1}\u_q \right ] + \beta^2 \M^{-1}\u_q \]
\end{theorem}

The proof of the above is given in Appendix \ref{proof:ode_peel_weight}.

\begin{theorem}\label{thm:srel_peel} 
We imagine the loss function in equation \ref{eq:loss_peel} to be over a linear predictor with weight matrix $\W$ whose rows are evolving as given in Theorem \ref{thm:ode_peel}. Then invoking the definition of $\srel$, as in Definition \ref{def:srel}, for $\x = (k,i)^{th}$-data point and $\x' = (c,j)^{th}-$data point we get, 
\begin{align}\label{eq:srel_peel}
\nonumber &\srel^2(t)_{(c,j),(k,i)} = \frac{ \norm{\W^+ \h_{c,j} - \W \h_{c,j}}^2}{\norm{\W^+ \h_{k,i} - \W \h_{k,i}}^2}\\
&= \frac{{\scriptsize \norm{\h_{c,j}^\top \cdot  \h_{k,i} }^2 {-} 2 \ip{\h_{k,i}}{\h_{c,j}} \cdot   \ip{\T_{k,i} \h_{c,j}}{\w_k}  +   \sum_{q=1}^K \ip{\T_{k,i} \h_{c,j}}{\w_q}^2}}{{\scriptsize \norm{  \h_{k,i} }^4 {-} 2 \tilde{h}_{k,i} \cdot \norm{\h_{k,i}}^2 \cdot \ip{\h_{k,i}}{\w_k}  + \tilde{h}_{k,i}^2  \cdot \sum_{q=1}^K \ip{\h_{k,i}}{\w_q}^2}} 
\end{align}
where we have defined $\tilde{h}_{k,i} \coloneqq \left [ \frac{\cN \lambda_1}{K} + \norm{\h_{k,i}}^2 \right ]$ and the matrices, $\R^{p \times p} \ni \T_{k,i} \coloneqq  \frac{\cN \lambda_1}{K} \I_p +  \h_{k,i} \cdot \h_{k,i}^\top$
\end{theorem} 
We note that the time dependence of the above expression is only via the $\{ \w_i \in \R^p \mid i =1,\ldots, K\}$, the row vectors of $\W$, which have been defined to be evolving as given in Theorem  \ref{thm:ode_peel}.
In the following subsection, we shall see an explicit instantiation of the above formula leading to a demonstration that non-trivial dynamics of $\srel$ is possible for even linear predictors when a regularization is used. 
\begin{proof}[{\bf of Theorem \ref{thm:srel_peel}}]
We can read off, that corresponding to sampling the $(k,i)^{th}-$data point and a choice of step-length $\eta$, for the $q^{th}-$row of $\W$ we would have, 
\[ \w_q^{+} = \w_q - \eta \cdot \pdv{\cL(\W;\h_{k,i})}{\w_q}  = \w_q - \eta \cdot \left ( -\delta_{k,q} \h_{k,i} + \T_{k,i} \w_q  \right ) \] 
As a consequence of the fictitious update above the norm squared of the change in the output at the training data point $(c,j)$ for some $c \in \{1,\ldots,K\}$ and $j \in \{1,\ldots,n_c\}$ is given as,
\begin{align*}
    &\norm{\W^+ \h_{c,j} - \W \h_{c,j}}^2
    =\sum_{q=1}^K \left ( \ip{\w_q^+}{\h_{c,j}} - \ip{\w_q}{\h_{c,j}} \right )^2\\
    =& \eta^2 \cdot \sum_{q=1}^K \ip{-\delta_{k,q} \h_{k,i} + \T_{k,i} \w_q}{\h_{c,j}}^2\\
    =& \eta^2 \cdot \sum_{q=1}^K  \left ( -\delta_{k,q} \ip{\h_{k,i}}{\h_{c,j}} + \h_{c,j}^\top \T_{k,i} \w_q  \right )^2\\
    =& \eta^2 \cdot \sum_{q=1}^K \left \{ \delta_{k,q} \left (  \ip{\h_{k,i}}{\h_{c,j}}^2 - 2 \cdot \ip{\h_{k,i}}{\h_{c,j}} \cdot  \h_{c,j}^\top \T_{k,i} \w_q \right )  \right \}\\
    +& \eta^2 \cdot \sum_{q=1}^K \left ( \h_{c,j}^\top \T_{k,i} \w_q \right )^2\\ 
    = &\eta^2 \left ( \ip{\h_{k,i}}{\h_{c,j}}^2 - 2 \cdot \ip{\h_{k,i}}{\h_{c,j}} \cdot  \h_{c,j}^\top \T_{k,i} \w_k  \right )\\
    +& \eta^2 \h_{c,j}^\top \T_{k,i}\left (  \sum_{q=1}^K \w_q \w_q^\top \right ) \T_{k,i}^\top \h_{c,j}\\
    =& \eta^2 \norm{\h_{c,j}^\top \cdot  \h_{k,i} }^2  -2\eta^2 \ip{\h_{k,i}}{\h_{c,j}} \cdot \ip{\T_{k,i} \h_{c,j}}{\w_k}\\
    +& \eta^2  \sum_{q=1}^K \ip{\T_{k,i} \h_{c,j}}{\w_q}^2\\
\end{align*}

In the last line above we have recalled that $\T_{k,i}$ is a symmetric matrix. 

Noting that $\T_{k,i} \h_{k,i} = \left [ \frac{\cN \lambda_1}{K} + \norm{\h_{k,i}}^2 \right ] \h_{k,i}$, we can read off from the above the final expression given in the theorem statement.  
\end{proof}




\subsection{Experimental Demonstration of Equation \ref{eq:srel_peel} on Random Neural Features}


For inter-class, we choose to work with $2$ classes and hence $K = 2$ (and $\W \in \R^{2 \times p}$ in the experiments) and we will always have $n_k = 10^3$ feature vectors per class and in all our experiments we fix $\lambda_1 = 1, \theta = 0.001$. For an experimental demonstration of the non-trivial time dynamics of equation \ref{eq:srel_peel}, we need to instantiate the $\{ \h_{k,i} \in \R^p \mid k = 1,\ldots, K ~\& i=1,\ldots,n_k \}$ feature vectors via a choice of model. We create the $\h_{k,i}$s as the outputs of a one hidden layer neural net mapping, $\net : \R^{\rm dim} \ni \x \mapsto \relu(\W_{r}\x) \in \R^p$ for some ${\rm dim} \in \Z^+$ whose $\W_r \in \R^{p \times {\rm dim}}$ matrix is randomly sampled from a Gaussian distribution. The input $\x$s to the above net are sampled from either of the two distributions, 

\begin{itemize}
    \item An isotropic Gaussian distribution whose along each coordinate the mean is $1$ and variance is $2$. These data are labelled with $\e_1 = (1,0)$
    \item An isotropic Gaussian distribution whose along each coordinate the mean is $9$ and variance is $1$. These data are labelled with $\e_2 = (0,1)$
\end{itemize}




\begin{figure}[htbp!]
    \centering
    \includegraphics[scale=0.3]{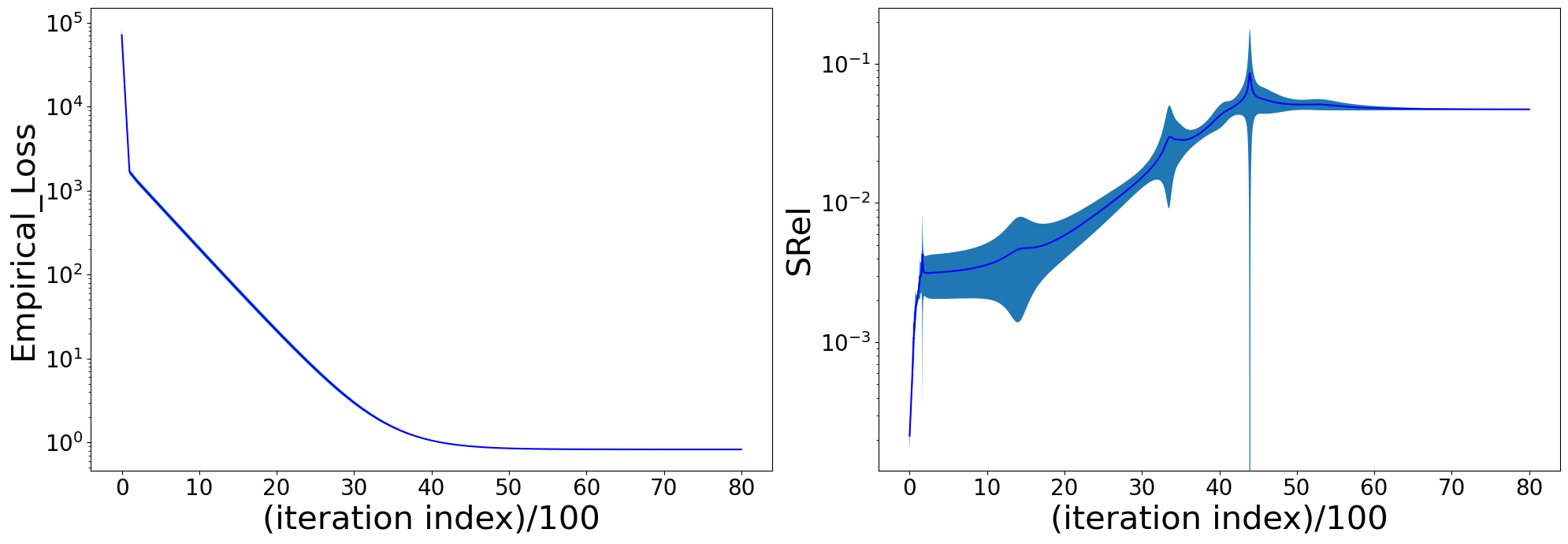}
    \caption{Variation of Empirical Loss (left) and interclass $S_{\rm rel}$ (right) respectively with the number of iterations for  $p=10, ~{\rm dim} = 100$}
    \label{fig:p_10_dim_100}
\end{figure}

\begin{figure}[htbp!]
    \centering
    \includegraphics[scale=0.3]{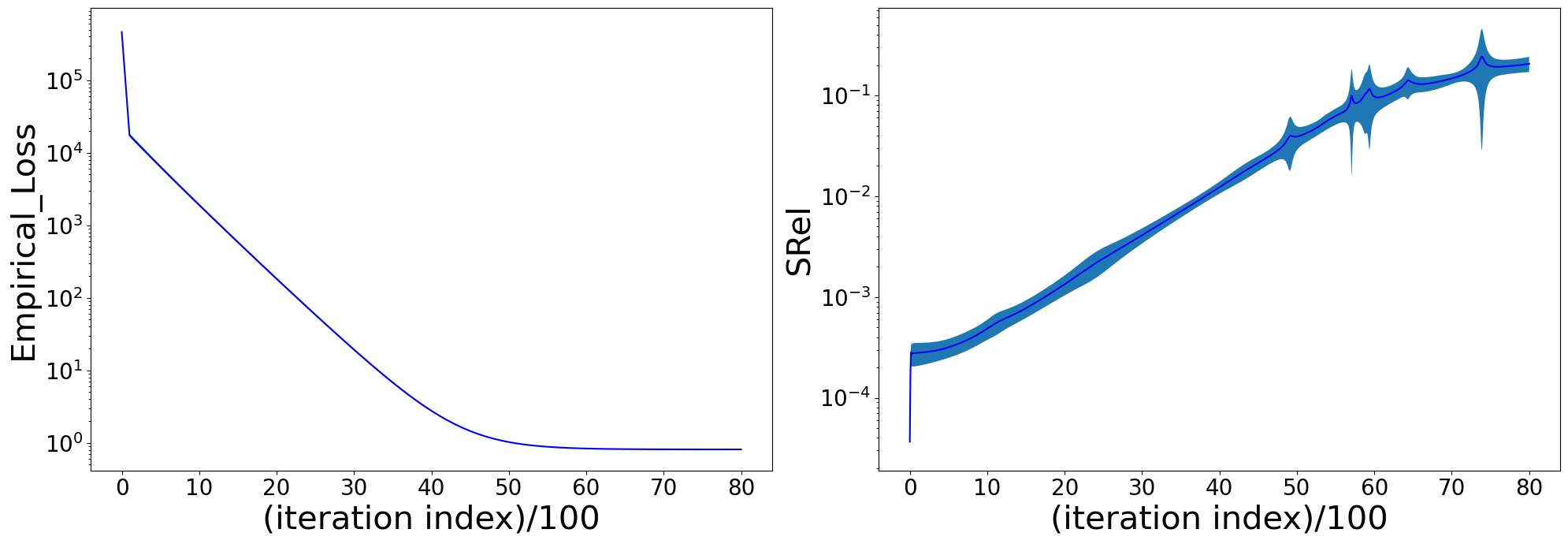}
    \caption{Variation of Empirical Loss (left) and interclass $S_{\rm rel}$ (right) respectively with the number of iterations for  $p=100, ~{\rm dim} = 100$}
    \label{fig:p_100_dim_100}
\end{figure}

\begin{figure}[htbp!]
    \centering
    \includegraphics[scale=0.3]{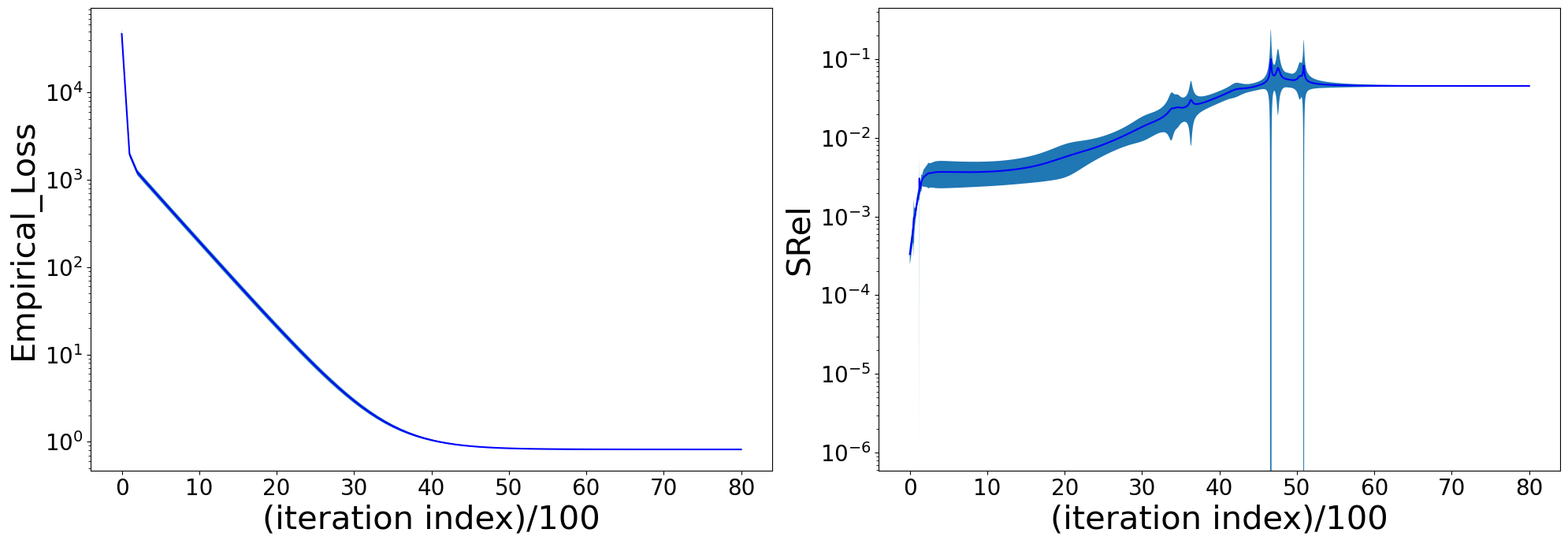}
    \caption{Variation of Empirical Loss (left) and intraclass $S_{\rm rel}$ (right) respectively with the number of iterations for  $p=10, ~{\rm dim} = 100$}
    \label{fig:p_10_dim_100_intraclass}
\end{figure}

\begin{figure}[htbp!]
    \centering
    \includegraphics[scale=0.3]{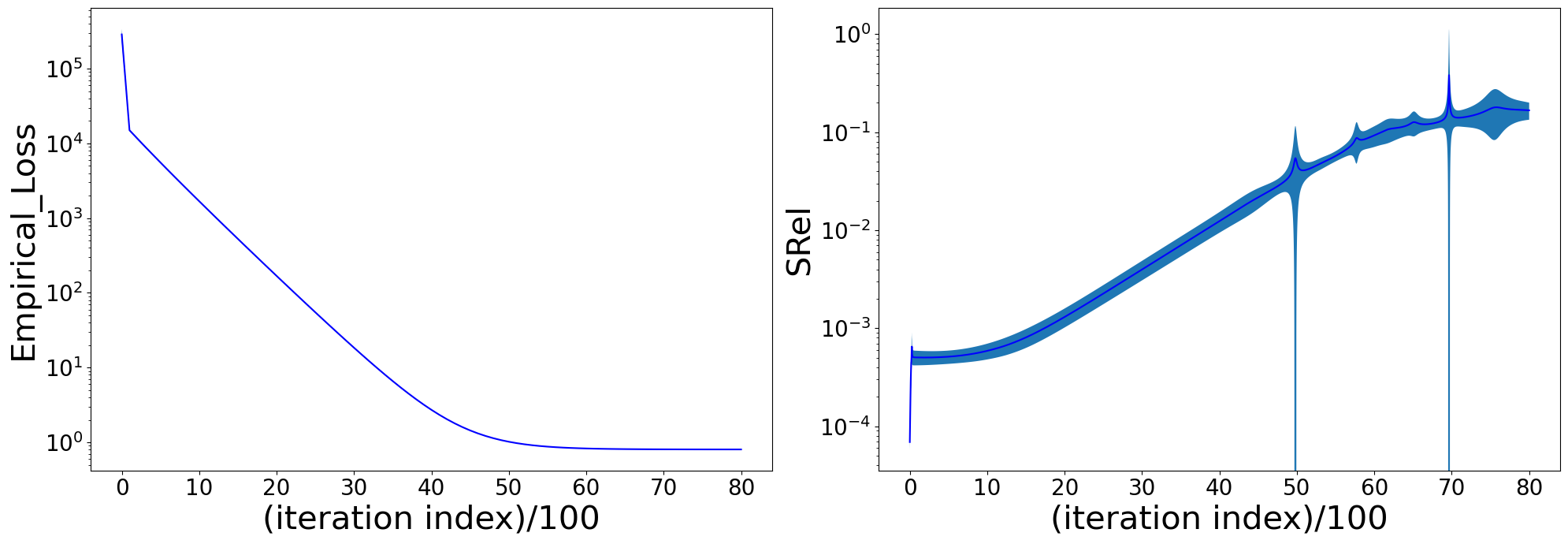}
    \caption{Variation of Empirical Loss (left) and intraclass $S_{\rm rel}$ (right) respectively with the number of iterations for  $p=100, ~{\rm dim} = 100$}
    \label{fig:p_100_dim_100_intraclass}
\end{figure}

Further, in computing equation \ref{eq:srel_peel}, we choose $\h_{k,i}$ and $\h_{c,j}$ as the outputs of the above random net $\net$ for an arbitrary pair of data points generated as above with labels $\e_1$ and $\e_2$ respectively. Whereas for inter-class, we consider data-points with labels $\e_1$ only.

In the right side plots of Figures \ref{fig:p_10_dim_100} and \ref{fig:p_100_dim_100} we see the time dynamics of inter-class $\srel$, while in Figures \ref{fig:p_10_dim_100_intraclass} and \ref{fig:p_100_dim_100_intraclass} we see the time dynamics of intra-class $\srel$, as given in equation \ref{eq:srel_peel} for the above setup for different values of $p$. The plots are averaged over multiple time evolutions starting from random initial values for the two rows of the $\W$ matrix in equation \ref{eq:srel_peel}, and over $20$ random samples of the data-points for smoothening.

From the figures we confirm the following two salient points, 

\begin{itemize}
    \item Firstly, we note the emergence of  the hypothesized two phase behaviour of $\srel$, that there is always an initial time interval and a semi-infinite late time phase s.t the value of $\srel$ in the later is always larger than during most of the former. 
    \item Alongside the evolution of $\srel$ in all these figures we have also shown the evolution of the empirical loss as given in equation \ref{eq:loss_peel} for the above setup and evaluated once every $100$ iterations. 
    
    Its clear from the figures that the initial phase when the empirical loss is falling rapidly is always the time when $\srel$ is rising rapidly. And in late times when the empirical loss has approximately stabilized, $\srel$ too has stabilized. 
\end{itemize}

We note that the dynamics of $\srel$ that emerged from the analytically tractable model proposed here is in close correspondence to what was seen in regression experiments with deep-learning as was exhibited in Section \ref{exp:reg}.

\section{Analytic Tracking of $S_{\rm rel}$ for Training a $\relu$ Gate by an ODE}\label{sec:relu_ode}

In \cite{mukherjee2020study}, the following algorithm was shown to train a $\relu$ gate (converge to a global minima of the risk) in linear time using minimal distributional conditions if the labels are exactly realizable by a ground-truth weight $\w_*$ of the $\relu$ gate.  

\begin{algorithm}[H]
	\caption{Modified S.G.D. for training a $\relu$ gate in the realizable setting}
	\label{dadushrelu}
	\begin{algorithmic}[1]
		\State {{\bf Input:} Sampling access to a distribution ${\cal D}$ on $\R^n$.}
		\State {{\bf Input:} Oracle access to labels $\R \ni y = \relu(\w_*^{\top}\x)$ when queried with some $\x \in \R^n$}  
		\State {{\bf Input:} An arbitrarily chosen starting point of $\w_1 \in \R^n$ and a step-length $\eta >0$}
		\For{$t = 1,\ldots$}
		\State {Sample $\x_t \sim {\cal D}$ and query the  oracle with it.} 
		\State {The oracle replies back with $y_t = \relu(\w_*^{\top}\x_t)$}
		\State {Form the gradient (proxy), 
		\[\g_t := -\mathbf{1}_{y_t > 0} (y_t - \w_t^\top \x_t)\x_t\]}
		\State {$\w_{t+1} := \w_t - \eta \g_t$}
		\EndFor
	\end{algorithmic}
\end{algorithm}

For some $\beta >0$, a reasonable choice for the O.D.E. form of the Algorithm \ref{dadushrelu} from \cite{mukherjee2020study} can be,  

\begin{definition}
\begin{equation}\label{relu_ode}
\begin{split}
&\dv{\w(t)}{t}\\
= &\beta \cdot \E_{\x \distas {\cal D}} \left [  \ind{ \ip{\w_*}{\x} >0} \Big (\max\{0, \ip{\w_*}{\x} \} - \ip{\w(t)}{\x} \Big ) \x \right ] 
\end{split}
\end{equation}
\end{definition}

\begin{theorem}\label{thm:relu_ODE}
If the matrix $\M$ defined as $\R^{ n \times n} \ni \M \coloneqq \E \left [ \ind{\w_*^\top \x > 0} \x \x^\top \right ]$ is P.D, then for the ODE given in equation \ref{relu_ode} we have, $\lim_{t \rightarrow \infty} \w(t) = \w_*$ 
\end{theorem}

\begin{proof}{{\bf of Theorem} \ref{thm:relu_ODE}} 
We define the following quantity towards writing the proof succinctly,  
\begin{gather}\label{Mz} 
    \R^n \ni \z \coloneqq \E \left [ \ind{\w_*^\top \x > 0} \max \{ 0, \w_*^\top \x \} \x  \right ]
\end{gather}
Then the equation \ref{relu_ode} can be written as,
\begin{equation}
\begin{split}
\dv{\w(t)}{t} + \beta \M\w = \beta \z \implies \dv{(e^{\beta \M t} \w)}{t} = e^{\beta \M t} \beta \z   
\end{split}
\end{equation} 

The above can be integrated to get,  $\w (t) = e^{-\beta \M t} \w(0) + \M ^{-1} \left [ \I - e^{-\beta \M t} \right ] \z$

From the above solution it follows for the given assumptions about $\M$ that, $\lim_{t \rightarrow \infty} \w(t) = \M^{-1}\z$. Now note that it follows from the definitions that $\M \w_* = \z$ and from the invertibility of $\M$ the conclusion follows.  
\end{proof} 

Now we can imagine the above ODE dynamics to be coming from a ``risk function" (${\cal R}$) and a loss function ($\ell$) s.t we have, 

\begin{equation}
\begin{split}
\dv{\w(t)}{t} &= - \nabla_{\w} {\cal R} = \beta \{ \z - \M \w \}\\
\nabla_{\w(t)} \ell(\w(t),\x) &= -\beta \ind{\w_*^\top \x >0} \left [ \max\{ 0, \w_*^\top \x\} - \w(t)^\top \x \right ] \x 
\end{split}
\end{equation} 

The above motivates the following definition as needed for defining $S_{\rm rel}$, 

\begin{equation}
\begin{split}
\w^+(t) &= \w(t) - \eta \nabla_{\w(t)} \ell (\w(t),\x)\\
&= \w(t) + \eta \beta \ind{\w_*^\top \x >0} \left [ \max \{0, \w_*^\top \x\} - \w(t)^\top \x \right ] \x 
\end{split}
\end{equation} 

Now we observe the following two distinctive behaviours of the above quantity,  

\begin{itemize}
    \item In Lemma \ref{lem:up_srel_relu} we will see a simple {\it time independent} sufficient criteria to emerge for a candidate $\x'$ i.e $\ip{\x'}{\x}$ being small, which would make the change in the prediction at $\x'$ to be small because of this fictitious update of $\w(t)^+$ based on a $\x$ which makes an acute angle with $\w_*$.
    \item In contrast at late times, we will show in Lemma \ref{lem:low_srel_relu} that $S_{\rm rel}$ cannot be arbitrarily small and we see that an easy sufficient condition emerges to ensure that $S_{\rm rel}$ is large i.e  $\abs{\ip{\x'}{\x}}$ being large for a fixed $\x$ and for both the points making an acute angle with $\w_*$. 
\end{itemize}

\begin{lemma}\label{lem:up_srel_relu}
\begin{equation*}
\begin{split}
&\abs{\max \{0, \ip{\w^+(t)}{\x'} \} - \max \{ 0, \ip{\w(t)}{\x'} \} }\\
&\leq \eta \beta \ind{\ip{\w_*}{\x} >0}\abs{\ip{\x}{\x'}} \abs{\ip{e^{-\beta \M t }(\w(0)-\w_*)}{\x}}
\end{split}
\end{equation*} 
\end{lemma}

Proof of the above has been given in Appendix \ref{lemma:relu_2}.

Hence in words it follows that if $\ip{\w_*}{\x}>0$ then for all times, $\abs{\ip{\x}{\x'}}$ being small is sufficient condition for the change induced at $\x'$ to be small when the update on $\w(t)$ is induced by sampling the data $\x$.

\begin{theorem}\label{lem:low_srel_relu}
If $\ip{\w_*}{\x'} >0$ and $\ip{\w_*}{\x} >0$, then $S_{\rm rel}$ has the following lower bound for late times,
\[ \exists ~t_* \text{ s.t } \frac{  \abs{\ip{\x}{\x'}} }{\norm{\x}^2 } \leq S_{\rm rel}, \forall t > t_* \]
\end{theorem}

Proof of the above has been given in Appendix \ref{theorem:relu_last}. 

If we think of $\x'$ and $\x$ being in the same class given that they lie on the same side of the hyperplane defined by $\w_*$, it follows that later in the training $S_{\rm rel}(\x,\x')$ can't be arbitrarily small for this $\relu$ training algorithm. 


Further in support of the above model of training a $\relu$ gate and the lower bound on $S_{\rm rel}$ obtained therein we empirically compare it to the dynamics of $S_{\rm rel}$ when using usual mini-batch S.G.D. to minimize the standard $\ell_2-$risk on it in the realizable setting. In particular consider the risk corresponding to $\w_* = (1,1,\ldots,1) \in \R^{10}$. In Fig. \ref{fig:sgd_relu}, we choose $\x = (10,10,\ldots,10)$ and $\x' = (\sqrt{200},\sqrt{200},\ldots,\sqrt{200})$, and we plot the dynamics of $S_{\rm rel}(\x,\x')$ over iterations and averaged over multiple S.G.D. runs - all of which converged to $\w_*$ for all practical purposes. 

\begin{figure}[h!]
    \centering
    \includegraphics[scale = 0.6]{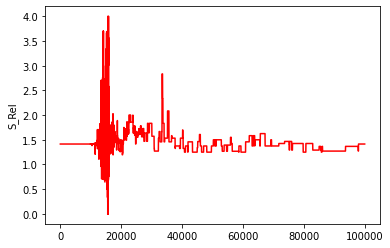}
	 \caption{Tracking $S_{\rm rel}$ (averaged over multiple runs of the algorithm) between a chosen pair of points for training a $\relu$ gate via mini-batch S.G.D. in $10$ dimensions in the realizable setting.} \label{fig:sgd_relu}
\end{figure}

Note that for the above values of $\w_*$, $\x$ and $\x'$, the late time lower bound on $S_{\rm rel}$ given in Theorem \ref{lem:low_srel_relu} applies and it evaluates to $\sim 1.4$ which is very close to the value around which the above plot is fluctuating. Such experiments can be conducted under various other parameter settings lending further credence to this model. {\it Thus here we again see that Hypothesis \ref{conj:large} bears out.}

Note that from combining equations \ref{numerator_1} and \ref{numerator_2} (in Appendix \ref{theorem:relu_last}) and their counterparts with $\x'$ replaced with $\x$ we can get an exact expression for $S_{\rm rel}$ for this situation of training a $\relu$ gate.


In the next section we shall move onto studying exactly provable behaviours of $S_{\rm rel}$ for certain polynomial hypothesis classes, which are strongly inspired from the study of large width neural networks.
 
\section{Analytic Tracking of $S_{\rm rel}$ for Gradient Flow on  $\ell_2-$ Loss on Degree$-d$ Weight Homogeneous Feature Linear Predictors}\label{sec:d_hom}    

From ongoing theoretical research in algorithmic regularization and gradient dynamics on extremely wide neural networks, \cite{bai2019beyond}, \cite{woodworth2020kernel} certain kinds of high-degree polynomials have come to focus as being good test cases of various neural phenomenon. Inspired from these, we define the following class of functions which shall be the predictors we focus on in this segment,  

\begin{definition}\label{d_hom}
For functions $\alpha : \R^n \rightarrow \R$ and $\bbeta_r : \R^n \rightarrow \R^n, r = 1,\ldots,w$ we defined the ``degree$-d$ weight homogeneous feature linear predictors" as, 
\[ f_{\rm d,w,lin}(\W,\x) \coloneqq \alpha(\x) + \sum_{r=1}^w \ip{\bbeta_r(\x)}{\w_r^d} \]
In above $\w_r$ is used to denote the $r^{th}-$row of $\W \in \R^{w \times n}$ and $\w_r^d \coloneqq (w_{r,1}^d,\ldots,w_{r,n}^d)$
\end{definition} 

{\remark 


As further motivation towards the above, we recall  Theorem $2.1$ in, \cite{lee2019wide}.
In there it was assumed that each layer of the net (say $\net$) was of size $n = $input-dimension. Hence each weight vector in each layer would be in $\R^n$ and these were in turn indexed as $\w_r$ for $r=1,\ldots,w$. Then, given a point $\W_0$, (usually the value at initialization), the gradient dynamics on $\net$ at asymptotically large widths was shown to be reproducible as gradient dynamics on a class of predictors which can be seen as a special case of Definition \ref{d_hom}, corresponding to,

\[ d = 1 \]
\[ \alpha(\x) = \net(\W_0,\x) -  \ip{\nabla_{\W} \net(\W,\x)  \mid_{\W = \W_0}}{\W_0}   \] 
\[ \bbeta_r(\x) = \nabla_{\w_r} \net(\W,\x)  \mid_{\W = \W_0} \]

The matrices $\W$ and $\W_0$ when occurring inside the inner product are to be understood as having been vectorized row-wise. 
}
        
We consider the $\ell_2-$loss evaluated on the above predictor for a data point $(\x,y)$,  

\begin{equation}\label{def:srel_1}
\begin{split}
\ell_{\rm d,w,lin}(\W,(\x,y)) \coloneqq  &\frac{1}{2} (y - f_{\rm d,w,lin}(\W,\x) )^2\\
&=  \frac{1}{2} \left (y - \alpha(\x) - \sum_{r=1}^w \ip{\bbeta_r(\x)}{\w_r^d}  \right )^2
\end{split}
\end{equation} 


 If we denote $\W_t$ as the $t^{th}$ iterate of a training algorithm on $\W$ with step-length $\eta$, then invoking Definition \ref{def:srel} in this context gives us, 

\begin{equation}\label{def:srel_2}
\begin{split}
\left [ S_{\rm rel} (t) \right ]_{\x', \x} \coloneqq  \frac{\abs{f_{\rm \rm d,w,lin}(\W_t^+,\x') - f_{\rm d,w,lin}(\W_t,\x')}}{ \abs{f_{\rm d,w,lin}(\W_t^+,\x) - f_{\rm d,w,lin}(\W_t,\x)}}\\
\text{where}~\forall r \in \{1,\ldots,w\}\\
\w_{r,t}^+(\x) = \w_{r,t} - \eta \nabla_{\w_{r,t}} \ell_{\rm d,lin} (\W_t, (\x,y))  
\end{split}
\end{equation}

Firstly, we demonstrate a time-independent upper bound for the above in the limit of small step-lengths, 

\begin{theorem}[{\bf A Sufficient Condition For the Continuous Time Value of $S_{\rm rel}$ to be Small}]\label{thm:srel_cont_upper}
\begin{equation*}
\begin{split}
&\lim_{\eta \rightarrow 0}  [S_{\rm rel}]_{\x',\x}\\
&\leq \left ( \max_{r=1,\dots,w} \norm{\bbeta_r(\x') \odot \bbeta_r(\x)} \right ) \cdot  \sum_{r=1}^w \frac{\norm{\w_{r,t}^{2(d - 1)}} }{\abs{\sum_{r=1}^w   \ip{ \bbeta_r(\x)^2}{\w_{r,t}^{2(d - 1)}} }}
\end{split}
\end{equation*}
\end{theorem}

Proof of the above theorem is given in Appendix \ref{proof: srel_cont_upper}. Thus from above we can read off the ``metric like" function induced on the data-space, 
\[ \funcP (\x,\x') =  \left ( \max_{r=1,\dots,w} \norm{\bbeta_r(\x') \odot \bbeta_r(\x)} \right ) \]
as envisaged more generally in Hypothesis \ref{conj:small}.

\subsection{Gradient Flow Dynamics for Training $f_{\rm d,w,lin}$}

Now consider solving the learning problem: $\min_{\W \in \R^{w \times n}} \E_{(\x,y)\sim {\cal D}} \left [  \ell_{\rm d,w,lin}(\W,(\x,y)) \right ]$ using the following gradient flow dynamics (which uses a parameter $\theta >0$), 

\begin{definition}[{\bf Defining a Gradient Flow Dynamics for Training $f_{\rm d,w,lin}$}]
\begin{equation}\label{ode} 
\begin{split}
&\forall r \in \{1,\ldots,w\}\\ 
&\dv{\w_r(t)}{t} \\
&= -\theta \pdv{}{\w_r(t)} \E_{(\x,y)\sim {\cal D}} \left [  \ell_{\rm d,w,lin}(\W,(\x,y)) \right ]\\
&= -\theta\pdv{}{\w_r(t)} \E_{(\x,y)} \left [ \frac{1}{2} \left (y - \alpha(\x) - \sum_{r=1}^w \ip{\bbeta_r(\x)}{\w_r(t)^d}  \right )^2 \right ]\\
&= \theta \cdot d \cdot \E_{(\x,y)} \Bigg [ \left (y - \alpha(\x) - \sum_{p=1}^w \ip{\bbeta_p(\x)}{\w_{p}(t)^d}  \right )\\
&\cdot \left (\bbeta_r(\x) \odot \w_{r}(t)^{d-1} \right) \Bigg ]
\end{split}
\end{equation} 
\end{definition} 

From the above dynamics we shall now be able to obtain {\em an exact expression for local elasticity for training over diagonal orthogonal quadratic features}.

\begin{theorem}\label{thm:exact_S_rel}

Assume $w=n$ and suppose $\exists ~a_q ~\& ~b_q \forall q = 1,\ldots,n$ s.t 
\[ a_q \coloneqq \E_{(\x,y)} \left [ \left (y - \alpha(\x) \right )\bbeta_{q,q} (\x) \right ] \] 
\[  b_{q} \delta_{p,q} \coloneqq \E_{(\x,y)} \left [ \bbeta_{q,q}(\x)\bbeta_{p,p}(\x)  \right ] ~\forall p = 1,\ldots, n  \]

Then for $d=2$, $\W = \diag(\w)$ {\rm for some } $\w \in \R^n$, for features s.t $a_r >0, \forall r \in \{1,\ldots,n\}$ and initial conditions s.t $w_r(0)^2 \in (0, \frac{a_r}{b_r} ) ~\forall r \in \{1,\ldots,n\}$ for training via the ODE in equation \ref{ode}, for $S_{\rm rel}$, as defined via equations \ref{def:srel_1} and \ref{def:srel_2}, we have,

\begin{equation}\label{exact:srel_d_2}
\begin{split}
\lim_{\eta \rightarrow 0}  [S_{\rm rel}]_{\x',\x} =  \frac{\abs{ \sum_{r=1}^n  \frac{a_r \bbeta_{r,r}(\x') \bbeta_{r,r}(\x)}{b_r + \left \{ -b_r + \frac{a_r}{w_r^2(0)} \right \}e^{-4 \theta a_r t}}  } }{\abs{\sum_{r=1}^n  \frac{a_r \bbeta_{r,r}(\x)^2}{b_r + \left \{ -b_r + \frac{a_r}{w_r^2(0)} \right \}e^{-4 \theta a_r t}}  }}\\
\end{split}
\end{equation}
\end{theorem}

Proof of the above theorem is in Appendix \ref{proof:exact_S_rel}.

{\remark Note that the the condition $a_r >0, \forall r$ can be easily imagined to be satisfied if we have (A) ``orthogonal features" i.e $\forall p \neq q, 0=\E_{(\x,y)} \left [  \bbeta_{p,p}(\x) \cdot \bbeta_{q,q}(\x) \right ]$ (as already assumed in above) and (B) a ``realizable" setting with $y = \alpha(\x) + \sum_{r=1}^w \theta_r^2 \beta_{r,r}(\x)$ for some constants $\theta_r$}

For intuition, consider a special case of the above with $w=n=3, \bbeta_{i,j}(\z) = z_i\delta_{i,j}, ~\forall i,j = \{1,2,3\}, \z \in \R^3, \x \distas {\cal N}(0,\I_{3 \times 3})$ and $y = \alpha(\x) + x_1 + 4 x_2 + 9 x_3$ for some $\alpha(\x)$. Then $~\forall i =1,2,3 ~b_i = 1$  Then $~\forall i =1,2,3 ~a_i = \E_{\x \distas {\cal N}(0,\I_{3 \times 3})} \left [ x_i \cdot ( x_1 + 4 x_2 + 9x_3)\right ]$ which implies $a_1 = 1, a_2 =4, a_3 = 9$. Then $\frac{b_1}{a_1} = 1, \frac{b_2}{a_2} = \frac{1}{4} ,\frac{b_3}{a_3} = \frac{1}{9}$. And hence we can consistently choose $w_1(0)^2 = \frac{1}{2}, w_2(0)^2 = 2,, w_3(0)^2 = 4$. Corresponding to the choices above, equation \ref{exact:srel_d_2} evaluates to,  

\begin{equation}\label{ex:s_rel}
\begin{split}
\lim_{\eta \rightarrow 0}  [S_{\rm rel}](t)_{\theta, \x',\x} =  \frac{\abs{\frac{x_1' x_1}{1+e^{-4\theta t}} +  \frac{x_2'x_2}{\frac{1}{4} + \frac{1}{4}e^{-16 \theta t}}  + \frac{x_3'x_3}{\frac{1}{9} + \frac{5}{36}e^{-36 \theta t}} }}{\abs{\frac{ x_1^2}{1+e^{-4\theta t}} +  \frac{x_2^2}{\frac{1}{4} + \frac{1}{4}e^{-16 \theta t}}  + \frac{x_3^2}{\frac{1}{9} + \frac{5}{36}e^{-36 \theta t}}}}
\end{split}
\end{equation}

We will consider two ways of visualizing this formula,  

\paragraph{{\it At all times $S_{\rm rel}$ decreases as $\x$ gets farther away from $\x'$} }
\begin{figure}[h!]
    \centering
    \includegraphics[scale = 0.6]{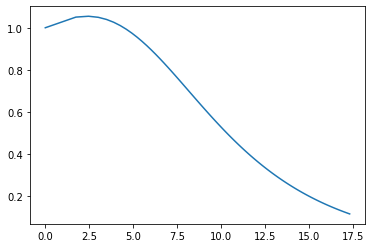}
	    \caption{In above we have chosen $\theta = 0.001$, $t=500$, $\x' = (-1,-11,1)$ and $\x(z) = \x' + (\sqrt{z},\sqrt{z},\sqrt{z})$ for $z=1,2,\ldots,100$ and then we have plotted $S_{\rm rel}$ (on the $y-$axis) as given in equation \ref{ex:s_rel} as a function of $\norm{\x' - \x(z)}$ (on the $x-$axis) which is increasing with  $z$}\label{fig:ex_srel_1}
\end{figure}	    

The key property we observe in Figure \ref{fig:ex_srel_1} is that $\exists$ a $\tilde{t}$ s.t $\forall t > \tilde{t}$ $S_{\rm rel}(\x,\x')$ decreases as the sampled point $\x$ is chosen farther away from the fixed testing point $\x'$. This property can be reproduced for various different paths of moving $\x$ away from $\x'$.  {\it We observe the conceptual resemblance of the above figure with Figure \ref{fig:non_real_srel_d}  as were observed on neural networks.}
    
\paragraph{{\it For all $(\x,\x')$ $S_{\rm rel}$ is a non-decreasing function of time for all later times}}

In Figure \ref{fig:ode} we plot equation \ref{ex:s_rel} as a function of time for certain pairs of $(\x,\x')$. Note that the initial time behaviour $S_{\rm rel}$ is very different between Figures \ref{fig:ex_srel_increase} and \ref{fig:ex_srel_ini_dec} but both of them demonstrate the validity of Hypothesis \ref{conj:large} and values of $t_1,t_2,t_*$ as in the hypothesis can be read off from the plots. 

\begin{figure}[h!]
    \centering 
\begin{subfigure}{0.50\textwidth}
  \includegraphics[width=1\textwidth]{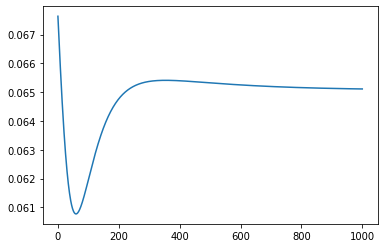}
  \caption{In above we have chosen $\x = (1,-1,1)$ and \\ $\x' = (1.01,0.999 ,1.2)$.}
  \label{fig:ex_srel_increase}
\end{subfigure}\hfil 
\begin{subfigure}{0.50\textwidth}
  \includegraphics[width=1\textwidth]{Figures/S_rel_vs_time_initial_decrease.png}
  \caption{In above we have chosen $\x = (1,-11,1)$ and \\
  $\x' = (1.01,0.999 ,1.2)$.}
  \label{fig:ex_srel_ini_dec}
\end{subfigure}
\caption{For $\theta = 0.001$, we plot the R.H.S of equation \ref{ex:s_rel} (y-axis) for $t \in \{ 1,2,\ldots,1000\}$ (x-axis)}
\label{fig:ode}
\end{figure}

More interestingly, we now cast the setup of Eq. \ref{ex:s_rel} as S.G.D. solving the following class of risk minimization questions, 

\begin{align}\label{def:sgd_risk} 
\nonumber &\min_{\w \in \R^3 } \E_{(\x,y) } \left [ \ell(\w,(\x,y)) \right ]\\
&\coloneqq  \min_{\w \in \R^3 } \E_{\x \distas {\cal N}(0,\I_{3 \times 3})} \left [ \frac{1}{2}\left (y(\x) - \left ( \alpha(\x) + \sum_{p=1}^3 \beta_{p,p}(\x) w_p^2 \right ) \right )^2  \right ]
\end{align}

with $\beta_{p,p}(\x) = x_p$ for $p=1,2,3$ as specified earlier and the true labels $y(\x)$ being also generated as earlier $y(\x) = \alpha(\x) + \ip {\x}{\w_*^2}$ with $\w_* = (1,2,3)$.We denote $(w_{t,1},w_{t,2},w_{t,3}) = \w_t \in \R^3$ as the $t^{th}-$iterate of the S.G.D. We initialize the S.G.D. at $w_{0,1} = \frac{1}{\sqrt{2}}, w_{0,2} = \sqrt{2}, w_{0,3} = 2$, the same initial conditions as used in deriving the O.D.E solution in equation \ref{ex:s_rel} and use a step-length of $\eta = 0.001$ as was the value of the $\theta$ parameter in the plots in Figure \ref{fig:ode}. Then the $t^{th}-$ step in the S.G.D. would read,
\[ \x_t \distas {\cal N}(0,\I_{3 \times 3}) \]
\[ \w_{t+1} = \w_t + 2\eta \left (  \sum_{p=1}^3 x_{t,p} (\w_{*,p}^2 - \w_{t,p}^2)\right ) \x_t \odot \w_t \] 

Thus at every $\w_t$ for the above algorithm we can compute $S_{\rm rel}(\x,\x')$ as given in Definition \ref{def:srel} with $\ell(\w_t,(\x,y))$ as given in Definition \ref{def:sgd_risk}. For the $\x ~\& ~\x'$ as in Figure \ref{fig:ode}, in Figure \ref{fig:sgd} we give the corresponding plots averaged over many samples of S.G.D. runs - and in all these samples the iterates  were all eventually satisfying the condition $\norm{\w_t - \w_*} < 10^{-6}$.

\begin{figure}[h!]
    \centering 
\begin{subfigure}{0.50\textwidth}
  \includegraphics[width=1\textwidth]{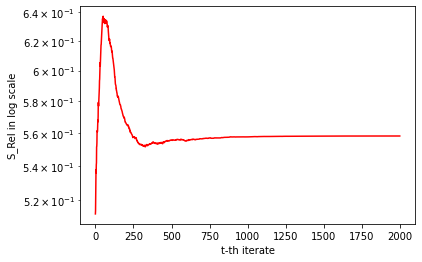}
  \caption{In above we have chosen $\x = (1,-1,1)$ \\ and $\x' = (1.01,0.999 ,1.2)$}
  \label{fig:sgd_ex_srel_increase}
\end{subfigure}\hfil 
\begin{subfigure}{0.50\textwidth}
  \includegraphics[width=1\textwidth]{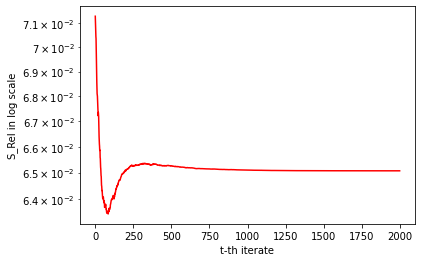}
  \caption{In above we have chosen $\x = (1,-11,1)$ \\ and $\x' = (1.01,0.999 ,1.2)$}
  \label{fig:sgd_ex_srel_ini_dec}
\end{subfigure}
\caption{ We plot $S_{\rm rel}$ for the $t^{th}-$iterate of S.G.D. training on the risk given in equation \ref{def:sgd_risk} for $t \in \{ 1,2,\ldots,2000\}$ and using step-length of $\eta = 0.001$. The plotted value is an average over multiple samples of S.G.D. runs. (The reader is urged to note the similarity in the hyper-parameter settings and the eventual similarity of the plots as in Figures \ref{fig:ode})}
\label{fig:sgd}
\end{figure}


The sharp resemblance (of values as well as of the shape) between Figures \ref{fig:ode} and \ref{fig:sgd}, justifies our approach of modelling the behaviour of local elasticity for realistic training algorithms like S.G.D. via appropriate limits of solutions of certain O.D.Es.  
~\\ \\
For certain set of special cases of equation \ref{exact:srel_d_2}, in the next theorem we shall prove some of the key features observed in these two figures.

\begin{theorem}[{\bf A Special Symmetric Case with $2$ Features in $2$ dimensions}]\label{thm:2_2_diag}
~\\
Suppose there are 2 features $\beta_1, \beta_2 : \R^2 \rightarrow \R^2$ and we are considering the following predictor class, 
\begin{equation}\label{exact_simplest}
\begin{split}
\R^{2} \times \R^2 \ni (\w , \x) \mapsto f_{\rm 2,2,lin}(\w,\x) \coloneqq \alpha(\x) + \sum_{r=1}^2 \beta_{r,r}(\x) \cdot w_r^2 \in \R 
\end{split}
\end{equation}
We assume the data-distribution over $\R^2 \times \R$ to be s.t the following are finite : $b_i \coloneqq \E_{(\x,y)} \left [  \beta_{i,i}(\x)^2 \right ], ~\forall i=1,2$ and $0 < a_i \coloneqq \E_{(\x,y)} \left [  (y - \alpha(\x)) \beta_{i,i}(\x) \right ], ~\forall i=1,2$
Further suppose that $\exists ~\alpha, \beta \in \R \setminus \{0\}$ s.t, 
\[ \frac{b_1}{a_1} = \frac{b_2}{a_2} =  \alpha^2, \quad  -\frac{b_1}{a_1} + \frac{1}{w_1(0)^2} = -\frac{b_2}{a_2} + \frac{1}{w_2(0)^2} = \beta^2 \]

And we suppose the following condition on the features at the point $\x$ and $\x'$ being considered,  
\[ \frac{\abs{\beta_{1,1}(\x') \beta_{1,1}(\x) + \beta_{2,2}(\x') \beta_{2,2}(\x)}}{\abs{\beta_{1,1}(\x)^2 + \beta_{2,2}(\x')^2}} > 0 \]
Then $\exists t^*$ s.t $\forall t \geq t^*$, $\lim_{\eta \rightarrow 0}  [S_{\rm rel}]_{\x',\x}$ (as defined via equations \ref{def:srel_1} and \ref{def:srel_2}) is always bounded away from $0$ and this lower bound is monotonically increasing with $t$. 
\end{theorem}

The proof of the above theorem can be found in Appendix \ref{proof:dhom:main}. We note that the above result can be seen as a partial evidence towards the property envisaged in Hypothesis \ref{conj:large}.

{\remark  For intuition consider a special case of the above with $\bbeta_{i,i}(\z) = z_i, ~\forall i = \{1,2\}$. Then we can write, 
\[ \frac{\abs{\beta_{1,1}(\x') \beta_{1,1}(\x) + \beta_{2,2}(\x') \beta_{2,2}(\x)}}{\abs{\beta_{1,1}(\x)^2 + \beta_{2,2}(\x')^2}} = \frac{\abs{x'_1x_1 + x'_2x_2}}{\abs{x_1^2 + x_2^2}} = \frac{\abs{\ip{\x'}{\x}}}{\norm{\x}^2}  \] 
Hence in this case the condition reduces to $\x'$ not being orthogonal to $\x$ for the the $S_{\rm rel}$ to be prevented from being arbitrarily low. 
}

\section{Conclusion}
In this work we were able to demonstrate via theory and neural experiments, a general phenomenon that the weight updates in many popular training setups cause prediction changes preferably for those data which are ``related" to the point being sampled for computing the stochastic gradient. And very crucially - and possibly surprisingly - our experiments in the classification setup showed that this preference sets in very early on in the training - much before the asymptotic classification accuracies have been attained. We recall that to pin down the above effect we had to introduce a new definition of ``local elasticity" (Definition \ref{def:srel_ce_k}) which measures the changes in the predicted distribution over classes (for a fixed input) while the training is underway. 

One of the main next steps from here on is to develop theoretical models where one can reproduce the phenomenon of local elasticity with neural classification experiments as demonstrated in Section \ref{exp:svhn}. More generally, it would be an exciting avenue of future research to identify learning scenarios where the extended idea of $\srel$ as given in Definition \ref{def:srel_ce_k} can be exactly calculated and thus be able to rigorously establish its interesting properties demonstrated in the experiments.  Lastly, a necessary and sufficient condition on a learning process doing regression remains to be identified for which one can rigorously prove the late time non-decreasing behavior of $\srel$. 

\section{Acknowledgements}{We would like to thank Weijie Su, Hangfeng He, Jiayao Zhang, Danny Wood and Neil Walton for various discussions while writing this paper. Anirbit Mukherjee would like to thank Wharton Dean’s Fund for Postdoctoral Research and Weijie Su's NSF CAREER DMS-1847415 for funding this research.   }

\printbibliography

\appendix





\onecolumn
\section{Proofs of Section \ref{sec:disc_features}}\label{app:peel_proof} 

\subsection{Proof of Theorem \ref{thm:ode_peel}}\label{proof:ode_peel_weight}


\begin{proof}
We note by directly differentiating the given evolution equation,

\begin{equation}
        \begin{split}
         \w_q(t) &= e^{-\theta^2\M t}\left(\w_q(0)-\beta^2\M^{-1}\u_q\right)+\beta^2\M^{-1}\u_q\\
         \Longrightarrow  \dv{\w_q}{t} &= (-\theta^2 \M)e^{-\theta^2\M t}\left(\w_q(0)-\beta^2\M^{-1}\u_q\right)\\
         &= - \theta ^2 \M e^{-\theta^2\M t}\w_q(0) + \theta^2 \beta^2 e^{-\theta^2\M t} \u_q \\
         &=  - \theta ^2 \M ( \w_q + \beta^2 \M^{-1} e^{-\theta^2\M t}  \u_q - \beta^2 \M^{-1} \u_q) + + \theta^2 \beta^2 e^{-\theta^2\M t} \u_q \\
         & = - \theta ^2 \M \w_q + \theta^2 \beta^2 \u_q 
        \end{split}
\end{equation}
Thus, the statement in Theorem \ref{thm:ode_peel} follows. 
\end{proof}

\section{Proofs of Section \ref{sec:relu_ode}}\label{app:relu} 

\subsection{Proof of Lemma \ref{lem:up_srel_relu}}\label{lemma:relu_2} 

\begin{proof}
\begin{equation}
    \begin{split}
      &\abs{ \max \{0, \ip{\w^+(t)}{\x'} \} - \max \{ 0, \ip{\w(t)}{\x'} \} }\\
    = &\big \vert \max \{0, \ip{\w(t)}{\x} + \eta \beta \ind{\w_*^\top \x >0} \left [ \max \{0, \w_*^\top \x\} - \w(t)^\top \x \right ] \ip{\x}{\x'} \} \\
    &- \max \{ 0, \ip{\w(t)}{\x'} \} \big \vert \\
    \leq & \eta \beta \ind{\ip{\w_*}{\x} >0}\abs{\ip{\x}{\x'}} \abs{\max \{0, \w_*^\top \x\} - \w(t)^\top \x}\\
    \leq & \eta \beta \ind{\ip{\w_*}{\x} >0}\abs{\ip{\x}{\x'}} \cdot \\
    &\abs{\max \{0, \w_*^\top \x\} - \ip{e^{-\beta \M t }\w(0)}{\x} - \ip{\w_*}{\x} + \ip{e^{-\beta \M t}\w_*}{\x}  }\\
    \leq & \eta \beta \ind{\ip{\w_*}{\x} >0}\abs{\ip{\x}{\x'}} \abs{\ip{e^{-\beta \M t }(\w(0)-\w_*)}{\x}} 
    \end{split}
\end{equation}
\end{proof}

\subsection{Proof of Theorem \ref{lem:low_srel_relu}}\label{theorem:relu_last}

\begin{proof}
Starting similarly as in the above proof we see that,  
\begin{equation}\label{numerator_1}
    \begin{split}
      \ip{\w^+(t)}{\x'} &= \ip{\w(t)}{\x'} + \eta \beta \ind{\w_*^\top \x >0} \left [ \w_* - \w(t) \right ]^\top \x  \cdot \ip{\x}{\x'}\\
      &= \ip{\w(t)}{\x'} + \eta \beta \ind{\w_*^\top \x >0} \left [ e^{-\beta \M t} (\w_* - \w(0)) \right ]^\top \x  \cdot \ip{\x}{\x'}\\
      &= \ip{\w_*}{\x'} + \ip{e^{-\beta \M t}(\w(0) - \w_*)}{\x'}\\
      &+ \eta \beta \ind{\w_*^\top \x >0} \left [ e^{-\beta \M t}(\w_* - \w(0)) \right ]^\top \x  \cdot \ip{\x}{\x'}\\
    \end{split}
\end{equation} 

Further let us define $\alpha_1 ~\& ~\alpha_2$ as, 
\begin{equation}\label{numerator_2}
    \begin{split}
    \alpha_1 &\coloneqq \ip{\w(t)}{\x'} = \ip{e^{-\beta \M t}(\w(0) - \w_*)}{\x'} + \ip{\w_*}{\x'}\\
    \alpha_2 &\coloneqq \eta \beta \ind{\w_*^\top \x >0} \left [ e^{-\beta \M t}(\w_* - \w(0)) \right ]^\top \x  \cdot \ip{\x}{\x'}  
    \end{split}
\end{equation} 

Hence we have, $\abs{\max\{0, \ip{\w(t)^+}{\x'}\} - \max \{0, \ip{\w(t)}{\x'}\}} = \abs{\max \{0,\alpha_1 + \alpha_2 \}-\max \{0,\alpha_1\}}$

Now we invoke the assumption that $\ip{\w_*}{\x'} >0$
and the fact that anything weighted by $e^{-\beta \M t}$ is monotonically decreasing in time to $0$ in absolute value to realize that $\exists t_* >0$ s.t $\forall t > t_*$ we have $\alpha_1 >0 ~\& \alpha_1 + \alpha_2 >0$. Hence,

\begin{equation}\label{numerator}
    \begin{split}
    \forall t>t_*, &\abs{\max\{0, \ip{\w(t)^+}{\x'}\} - \max \{0, \ip{\w(t)}{\x'}\}}\\
    &= \abs{\alpha_2} = \eta \beta \abs{\ip{e^{-\beta \M t} (\w_* - \w(0)}{\x}}  \abs{\ip{\x}{\x'}} 
    \end{split}
\end{equation}     
    
We further invoke Lemma \ref{lem:up_srel_relu}  to get, \[ \abs{\max\{0, \ip{\w(t)^+}{\x}\} - \max \{0, \ip{\w(t)}{\x}\}} \leq \eta \beta \ind{\ip{\w_*}{\x} >0}\norm{\x}^2 \abs{\ip{e^{-\beta \M t }(\w(0)-\w_*)}{\x}} \] Combining this with equation \ref{numerator} we get the following lower bound for $S_{\rm rel}$, 

\begin{equation} 
    \begin{split}
     \frac{\eta \beta \abs{\ip{e^{-\beta \M t} (\w_* - \w(0)}{\x}}  \abs{\ip{\x}{\x'}} }{\eta \beta \ind{\ip{\w_*}{\x} >0}\norm{\x}^2 \abs{\ip{e^{-\beta \M t }(\w(0)-\w_*)}{\x}}} \leq S_{\rm rel}, \forall t > t_*
    \end{split}
\end{equation}     

Invoking the assumption that $\ip{\w_*}{\x} >0$ we can write the more succinct lower bound as stated in the lemma, 

\begin{equation} 
    \begin{split}
    \frac{  \abs{\ip{\x}{\x'}} }{\norm{\x}^2 } \leq S_{\rm rel}, \forall t > t_*,
    \end{split}
\end{equation}

\end{proof}
 
\section{Proofs of Section \ref{sec:d_hom}}\label{proof:d_hom} 

\subsection{Proof of Theorem \ref{thm:srel_cont_upper}}\label{proof: srel_cont_upper}

\begin{proof}
Simplifying equations \ref{def:srel_1} and \ref{def:srel_2} we get, 

\begin{equation}
\begin{split}
&\w_{r,t}^+(\x) - \w_{r,t} = - \eta \nabla_{\w_{r,t}} \ell_{\rm d,w,lin} (\W_t, (\x,y))\\
&= \eta \left (y - \alpha(\x) - \sum_{p=1}^w \ip{\bbeta_p(\x)}{\w_{p,t}^d}  \right ) \cdot d \cdot \left (\bbeta_r(\x) \odot \w_{r,t}^{d-1} \right)
\end{split}
\end{equation}

Hence, 

\begin{equation}
\begin{split}
&\w_{r,t}^{+d} - \w_{r,t}^{d}\\
=& \left (\w_{r,t} + \eta \left (y - \alpha(\x) - \sum_{p=1}^w \ip{\bbeta_p(\x)}{\w_{p,t}^d}  \right ) \cdot d \cdot \left (\bbeta_r(\x) \odot \w_{r,t}^{d-1} \right) \right )^d - \w_{r,t}^{d}\\
=& \sum_{k=1}^d \binom{d}{k} \left ( \eta \cdot d \right )^k \cdot \left (y - \alpha(\x) - \sum_{p=1}^w \ip{\bbeta_p(\x)}{\w_{p,t}^d}  \right )^k \cdot \w_{r,t}^{d-k} \odot \left (\bbeta_r(\x) \odot \w_{r,t}^{d-1} \right)^k
\end{split}
\end{equation}

From the Definition \ref{d_hom} it follows that, 

\begin{equation}
\begin{split}
&f_{\rm d,w,lin}(\W_t^+,\x') - f_{\rm d,w,lin}(\W_t,\x')\\
&=  \sum_{r=1}^w  \ip{\bbeta_r(\x')}{\w_{r,t}^{+d} - \w_{r,t}^{d}}\\
&= \sum_{r=1}^w \sum_{k=1}^d \binom{d}{k} \left ( \eta \cdot d \right )^k \cdot \left (y - \alpha(\x) - \sum_{p=1}^w \ip{\bbeta_p(\x)}{\w_{p,t}^d}  \right )^k \cdot \ip{\bbeta_r(\x') \odot \bbeta_r(\x)^k}{\w_{r,t}^{d(k+1) - 2k}}
\end{split}
\end{equation}

In the second equality above we have used the following fact, 

\begin{equation}
\begin{split}
\ip{\bbeta_r(\x')}{\w_{r,t}^{d-k} \odot \left (\bbeta_r(\x) \odot \w_{r,t}^{d-1} \right)^k} &= \sum_{i=1}^n \bbeta_{r,i}(\x') \cdot \w_{r,t,i}^{d-k} \cdot \bbeta_{r,i}(\x)^k \cdot \w_{r,t,i}^{k(d-1)}\\
&= \ip{\bbeta_r(\x') \odot \bbeta_r(\x)^k}{\w_{r,t}^{d(k+1)-2k}}
\end{split}
\end{equation}

Similar to the above we can read off, 

\begin{equation}
\begin{split}
&f_{\rm d,w,lin}(\W_t^+,\x) - f_{\rm d,w,lin}(\W_t,\x)\\
&= \sum_{r=1}^w \sum_{k=1}^d \binom{d}{k} \left ( \eta \cdot d \right )^k \cdot \left (y - \alpha(\x) - \sum_{p=1}^w \ip{\bbeta_p(\x)}{\w_{p,t}^d}  \right )^k \cdot \ip{\bbeta_r(\x)^{k+1}}{\w_{r,t}^{d(k+1) - 2k}}
\end{split}
\end{equation}

\paragraph{The limit of $\eta \rightarrow 0$} 
~\\ \\
Now we note that, 

\begin{equation}
\begin{split}
&\lim_{\eta \rightarrow 0} \frac{1}{\eta} \cdot \abs{f_{\rm d,w,lin}(\W_t^+,\x') - f_{\rm d,w,lin}(\W_t,\x')}\\
&= \abs{ \sum_{r=1}^w d^2 \cdot \left (y - \alpha(\x) - \sum_{p=1}^w \ip{\bbeta_p(\x)}{\w_{p,t}^d}  \right ) \cdot \ip{\bbeta_r(\x') \odot \bbeta_r(\x)}{\w_{r,t}^{2(d -1)}} }
\end{split}
\end{equation}

And similarly, 

\begin{equation}
\begin{split}
&\lim_{\eta \rightarrow 0} \frac{1}{\eta} \cdot \abs{f_{\rm d,w,lin}(\W_t^+,\x) - f_{\rm d,w,lin}(\W_t,\x)}\\
&= \abs{ \sum_{r=1}^w d^2 \cdot \left (y - \alpha(\x) - \sum_{p=1}^w \ip{\bbeta_p(\x)}{\w_{p,t}^d}  \right ) \cdot \ip{\bbeta_r(\x)^2}{\w_{r,t}^{2(d -1)}} }
\end{split}
\end{equation}

If $S_{\rm rel}$ is well defined then we can write, 

\begin{equation}\label{limit_S_rel}
\begin{split}
\lim_{\eta \rightarrow 0}  [S_{\rm rel}]_{\x',\x} &= \lim_{\eta \rightarrow 0} \frac{\frac{1}{\eta} \cdot \abs{f_{\rm d,w,lin}(\W_t^+,\x') - f_{\rm d,w,lin}(\W_t,\x')}}{\frac{1}{\eta} \cdot \abs{f_{\rm d,w,lin}(\W_t^+,\x) - f_{\rm d,w,lin}(\W_t,\x)}}\\
&= \frac{\abs{ \sum_{r=1}^w  d^2 \cdot \left (y - \alpha(\x) - \sum_{p=1}^w \ip{\bbeta_p(\x)}{\w_{p,t}^d}  \right ) \cdot \ip{\bbeta_r(\x') \odot \bbeta_r(\x)}{\w_{r,t}^{2(d - 1)}} }}{\abs{ \sum_{r=1}^w   d^2 \cdot \left (y - \alpha(\x) - \sum_{p=1}^w \ip{\bbeta_p(\x)}{\w_{p,t}^d}  \right ) \cdot \ip{\bbeta_r(\x) \odot \bbeta_r(\x)}{\w_{r,t}^{2(d - 1)}} }}\\
&= \frac{\abs{ \sum_{r=1}^w  \ip{\bbeta_r(\x') \odot \bbeta_r(\x)}{\w_{r,t}^{2(d - 1)}} }}{\abs{\sum_{r=1}^w   \ip{ \bbeta_r(\x)^2}{\w_{r,t}^{2(d - 1)}} }}
\end{split}
\end{equation}


Hence from equation \ref{limit_S_rel} we have the expression that we set out to prove, 

\begin{equation}
\begin{split}
&\lim_{\eta \rightarrow 0}  [S_{\rm rel}]_{\x',\x} \leq \sum_{r=1}^w \norm{\bbeta_r(\x') \odot \bbeta_r(\x)} \cdot \frac{\norm{\w_{r,t}^{2(d - 1)}} }{\abs{\sum_{r=1}^w   \ip{ \bbeta_r(\x)^2}{\w_{r,t}^{2(d - 1)}} }}\\
\implies &\lim_{\eta \rightarrow 0}  [S_{\rm rel}]_{\x',\x} \leq \left ( \max_{r=1,\dots,w} \norm{\bbeta_r(\x') \odot \bbeta_r(\x)} \right ) \cdot  \sum_{r=1}^w \frac{\norm{\w_{r,t}^{2(d - 1)}} }{\abs{\sum_{r=1}^w   \ip{ \bbeta_r(\x)^2}{\w_{r,t}^{2(d - 1)}} }}
\end{split}
\end{equation}
\end{proof}

\subsection{Proof of Theorem \ref{thm:exact_S_rel}}\label{proof:exact_S_rel}

\begin{proof}
Upon invoking the assumptions $w =n ~\& ~\W = \diag(\w)$ {\rm for some } $\w \in \R^n$ and denoting the $q^{th}$ coordinate of $\w$ to be $w_q$, equation \ref{ode} will reduce to, 

\begin{equation}
\begin{split}
&\forall q \in \{1,\ldots,n\}\\
\frac{1}{\theta d} \cdot \dv{w_q(t)}{t} &= \E_{(\x,y)} \left [ \left ( y - \alpha(\x) - \sum_{p=1}^n \bbeta_{p,p}(\x) \cdot w_{p}(t)^d  \right ) \cdot \left (\bbeta_{q,q}(\x) \cdot w_{q}(t)^{d-1} \right) \right ]\\
\end{split}
\end{equation} 

We recall the definition of $a_q$ and we define a new set of distribution dependent constants $b_{p,q}$ as, 

\[ \forall q \in \{1,\ldots,n\} \]
\[  a_q =  \E_{(\x,y)} \left [ \left (y - \alpha(\x) \right )\bbeta_{q,q} (\x) \right ]  ~\& \left \{ b_{q,p} \coloneqq \E_{(\x,y)} \left [ \bbeta_{q,q}(\x)\bbeta_{p,p}(\x)  \right ] \mid p = 1,\ldots, n \right \}  \]

Then the ODE system can be re-written as $\forall q \in \{1,\ldots,n\}$, 

\[ \frac{1}{\theta d} \cdot \dv{w_q(t)}{t} = a_q \cdot w_{q}(t)^{d-1} - \left ( \sum_{p=1}^n b_{p,q} \cdot w_{p}(t)^d \right )\cdot w_q(t)^{d-1} \]

which can be rearranged to, 

\begin{equation}
\begin{split}
\frac{\dd w_q(t)}{a_q \cdot w_{q}(t)^{d-1} - \left ( \sum_{p=1}^n b_{p,q} \cdot w_{p}(t)^d \right )\cdot w_q(t)^{d-1}} &= \theta \cdot d \cdot  \dd{t} 
\end{split}
\end{equation}

\paragraph{Our assumption can be written as : $b_{p,q} = b_q \delta_{p,q}$ for some constants $b_q, q = 1,\ldots,n$}
~\\ 
Then the above ODE reduces to, 

\begin{equation}\label{ode_diag}
\begin{split}
\forall q &\in \{1,\ldots,n\}\\
\implies \frac{\dd w_q(t)}{a_q \cdot w_{q}(t)^{d-1} - b_q \cdot w_q(t)^{2d-1} } &= \theta \cdot d \cdot  \dd{t}
\end{split}
\end{equation}

\paragraph{Now we invoke the assumption that : $d=2$ and $a_q >0 ~\& ~w_q(0)^2 \in (0, \frac{a_q}{b_q} ) \forall ~q \in \{1,\ldots,n\}$}

Then the ODEs can be exactly integrated as follows,

\begin{equation}
\begin{split}
\forall q &\in \{1,\ldots,n\}\\
\implies \int_{w_q(0)}^{w_q(t)} \frac{\dd w_q(t)}{a_q \cdot w_{q}(t) - b_q \cdot w_q(t)^{3} } &= 2\theta  t\\
\implies \frac{1}{2 a_q} \cdot \left [ \log \frac{w_q(t)^2}{a_q - b_q w_q(t)^2} - \log \frac{w_q(0)^2}{a_q - b_q w_q(0)^2} \right ] &=  2\theta t\\
\implies w_q^2(t) = \frac{a_q}{b_q + \left \{ -b_q + \frac{a_q}{w_q^2(0)} \right \}e^{-4 \theta a_q t}}
\end{split}
\end{equation}

Note that in the intermediate steps above we need $0 < w_q(t)^2, w_q(0)^2 < \frac{a_q}{b_q}$ for the logarithms to be well defined. And $w_q(t)^2 < \frac{a_q}{b_q}$ is ensured for the solution if we have $a_q >0 ~\&  -b_q + \frac{a_q}{w_q^2(0)} >0 \Leftrightarrow w_q(0)^2 < \frac{a_q}{b_q}$ and both of these are ensured by the assumptions.    

Now for this case of $\W = \diag(\w) \in \R^{n \times n} ~\& ~d=2$, equation \ref{limit_S_rel} along with the ODE solution given above yields the following expression as claimed in the theorem statement, 

\begin{equation}\label{exact_s_rel}
\begin{split}
\lim_{\eta \rightarrow 0}  [S_{\rm rel}]_{\x',\x} = \frac{\abs{ \sum_{r=1}^n  \bbeta_{r,r}(\x') \bbeta_{r,r}(\x) \cdot w_{r,t}^{2}} }{\abs{\sum_{r=1}^n   \bbeta_{r,r}(\x)^2 \cdot w_{r,t}^{2} }} = \frac{\abs{ \sum_{r=1}^n  \frac{a_r \bbeta_{r,r}(\x') \bbeta_{r,r}(\x)}{b_r + \left \{ -b_r + \frac{a_r}{w_r^2(0)} \right \}e^{-4 \theta a_r t}}  } }{\abs{\sum_{r=1}^n  \frac{a_r \bbeta_{r,r}(\x)^2}{b_r + \left \{ -b_r + \frac{a_r}{w_r^2(0)} \right \}e^{-4 \theta a_r t}}  }}\\
\end{split}
\end{equation}
\end{proof}

\subsection{Proof of Theorem \ref{thm:2_2_diag}}\label{proof:dhom:main}
\begin{proof}
Define, 
\[ k_1 \coloneqq \frac{\beta_{1,1}(\x') \beta_{1,1}(\x)}{\beta_{1,1}(\x') \beta_{1,1}(\x) + \beta_{2,2}(\x') \beta_{2,2}(\x)} \] 
\[ k_2 \coloneqq \frac{\beta_{1,1}(\x)^2}{\beta_{1,1}(\x)^2 + \beta_{2,2}(\x')^2}  \]

Using the assumptions given in Theorem \ref{thm:2_2_diag} we can simplify equation \ref{exact_s_rel} in this case to get, 

\begin{equation}\label{d_hom_s_rel_exact_22}
\begin{split}
\lim_{\eta \rightarrow 0}  [S_{\rm rel}]_{\x',\x} = \frac{\abs{\beta_{1,1}(\x') \beta_{1,1}(\x) + \beta_{2,2}(\x') \beta_{2,2}(\x)}}{\abs{\beta_{1,1}(\x)^2 + \beta_{2,2}(\x')^2}} \cdot \frac{\abs{\alpha^2 + \beta^2 \left ( k_1 e^{-4 a_2 \theta  t} + (1-k_1)e^{-4 a_1 \theta  t} \right ) }}{\abs{\alpha^2 + \beta^2 \left ( k_2 e^{-4 a_2 \theta  t} + (1-k_2)e^{-4 a_1 \theta  t} \right ) }}
\end{split}
\end{equation} 

Note that $k_2 \geq 0$ and hence $\beta^2 k_2$ can be mapped to the parameter $b_2$ in Lemma \ref{bound}. Also either $k_1$ or $1-k_1$ is always non-negative and hence the non-negative one of the two multiplied with $\beta^2$ maps to the parameter $b_1$ in Lemma \ref{bound}. Since $a_1,a_2 > 0$ and $\theta >0$, the parameters $p$ and $q$ in Lemma \ref{bound} get identified. Now its easy to see that we can invoke the lemma to realize that $\exists t_1^*,t_2^* >0$ s.t for all $t \in [\max\{t_1^*,t_2^*\},\infty)$ we have, 

\begin{equation}
\begin{split}
0 < \frac{\abs{\beta_{1,1}(\x') \beta_{1,1}(\x) + \beta_{2,2}(\x') \beta_{2,2}(\x)}}{\abs{\beta_{1,1}(\x)^2 + \beta_{2,2}(\x')^2}} \cdot \frac{\alpha^2  - \abs{b_1}e^{-q^2t_2^*}}{\alpha^2 + \abs{b_2}e^{-p^2t} + \beta^2e^{-q^2t}} \leq \lim_{\eta \rightarrow 0}  [S_{\rm rel}]_{\x',\x} 
\end{split}
\end{equation}

In above we have used the assumption that the lower bound coming from Lemma \ref{bound} is strictly above $0$ and also the assumption about the first factor on the RHS of equation \ref{d_hom_s_rel_exact_22}  being strictly above $0$.  Hence for gradient flow on squared loss on the predictor class given in equation \ref{exact_simplest}, $\lim_{\eta \rightarrow 0}  [S_{\rm rel}]_{\x',\x}$ is s.t it can't decrease below a certain constant for all times beyond a certain initial time  interval. 
\end{proof}

\begin{lemma}\label{bound}
Let, 
\[ \funcf (t) \coloneqq \frac{\abs{\alpha^2 + b_1e^{-p^2t} + c_1e^{-q^2t}}}{\abs{\alpha^2 + b_2e^{-p^2t} + c_2e^{-q^2t}}}  \] 

Let $t_1^* >0$ be s.t $\forall t > t_1^*$ we have, $0 < \alpha^2 + b_2e^{-p^2t} + c_2e^{-q^2t}$

Let $t_2^* >0$ be s.t we have, $0 < \alpha^2  - b_1e^{-q^2t_2^*}$

Suppose $b_1,b_2 \geq 0$ and $\exists ~\beta \in \R \text{ s.t } b_1 +c_1 = b_2 + c_2 = \beta^2$

Then we have for all $t \in [\max\{t_1^*,t_2^*\},\infty)$, 
\[ \frac{\alpha^2  - \abs{b_1}e^{-q^2t_2^*}}{\alpha^2 + \abs{b_2}e^{-p^2t} + \beta^2e^{-q^2t}} \leq \funcf (t)  \] 
\end{lemma}

{\remark Note that because $\alpha^2 >0$ the above $t_1^* ~\&  ~t_2^*$ always exists.} 


\begin{proof}

So we have for all $t \in [t_1^*,\infty)$,  
\begin{equation}
\begin{split}
0 < \alpha^2 + b_2e^{-p^2t} + c_2e^{-q^2t} = \alpha^2 + \abs{b_2}e^{-p^2t} + (\beta^2 - \abs{b_2})e^{-q^2t} < \alpha^2 + \abs{b_2}e^{-p^2t} + \beta^2e^{-q^2t}
\end{split}
\end{equation} 

Similarly we have for all $t \in [t_2^*,\infty)$, 

\[ \alpha^2 + b_1e^{-p^2t} + c_1e^{-q^2t} = \alpha^2 + \abs{b_1}(e^{-p^2t} - e^{-q^2t}) + \beta^2e^{-q^2t} \geq (\alpha^2  - \abs{b_1}e^{-q^2t_2^*}) > 0\] 

Hence combining the above for all $t \in [\max\{t_1^*,t_2^*\},\infty)$ we have the required inequality. 
\end{proof}

\section{Additional Experiments}\label{app:exp} 
\subsection{A Comparative Study on CIFAR-100 of $\srel$ Definitions \ref{def:srel} and \ref{def:srel_ce_k}} 
\label{exp:cifar100resnet} 

CIFAR-100 is a labeled subset of the 80 million tiny-images dataset. It consists of 100 classes and 20 super-classes. Thus, each image has a coarse label and a fine label and in this paper, we only use the latter. There are 500 training images and 100 testing images per class. We repeat on this dataset the same comparative study of the new and the original definition of $\srel$, as was done in Section \ref{sec:svhn}. Some further implementation details are as follows, 


\begin{table}[tbh]
  \begin{center}
    \begin{tabular}{|l|r|}
   \hline     
      Parameter & Value  \\
   \hline
   Depth of  ResNet &  18 (no batch norm)\\
   The map that the $\net$ implements & $\R^{3 \times 32 \times 32} \to \R^{10}$\\ 
   \hline 
      Optimizer & ADAM \\
      Learning rate & $1\cdot 10^{-4}$\\
      Mini-batch size & 25\\
      Dropout or any other regularization & Not Used\\
      \hline
      Top $3$ classes for CIFAR-100 & motorcycle, orange, wardrobe \\
  \hline
  \end{tabular}
  \end{center}
  \label{table_cifar1003_no_bn}
\end{table}

On the left column of Figure \ref{fig:srel_cifar100} we see the time evolution of $\srel^{\rm k, smooth}$ for all the $9$ possible class pairs corresponding to images from the classes, $\lbrace {\rm  motorcycle, orange, wardrobe} \rbrace$ -- which were the $3$ classes of CIFAR-100 where ResNet-18 had the best accuracy. And the in the right column of Figure \ref{fig:srel_cifar100} we see the same for the original definition of $\srel$.

The plots show the mean $\srel$ (for both the definitions) value averaged over $5$ runs of the stochastic training algorithm as well as the standard deviation. Again, we observe that for all pairs, the current proposal results in a larger value when computed intra-class than when computed inter-class  - and it remains so at almost all times. As seen in the experiments in Section \ref{sec:svhn}, the original definition fails to pick up any such separation. 






\begin{figure}[h]
    \centering 
\begin{subfigure}{0.50\textwidth}
  \includegraphics[width=1\textwidth]{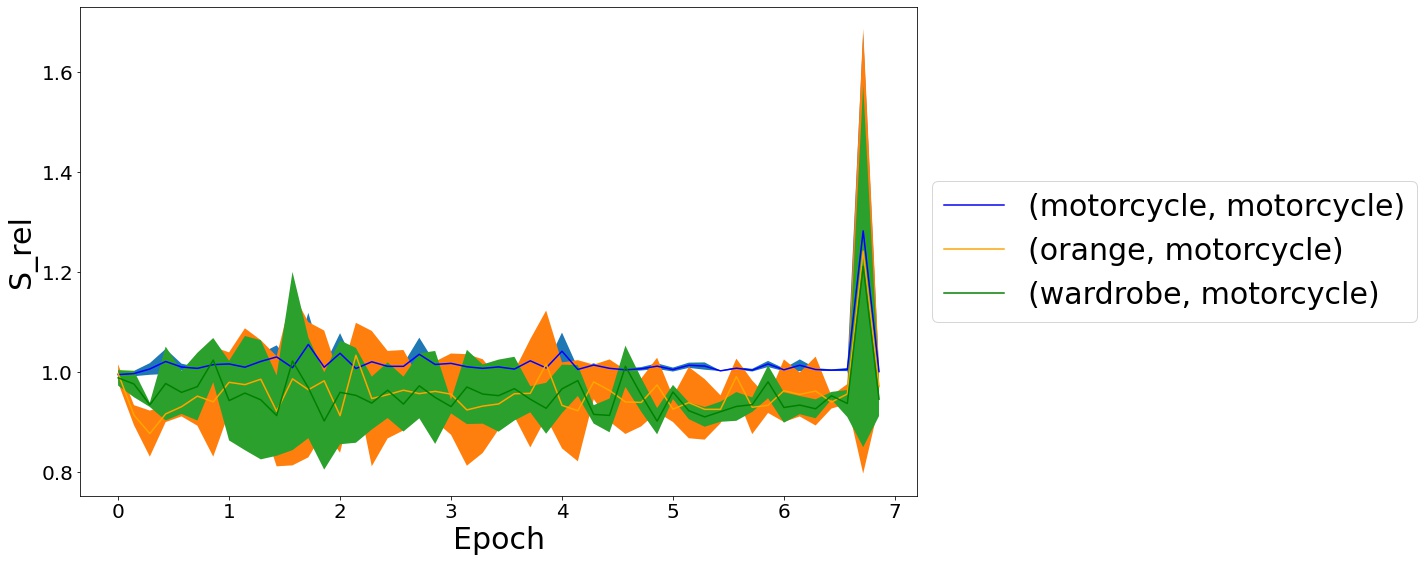}
  \caption{}
  \label{fig:13a}
\end{subfigure}\hfil 
\begin{subfigure}{0.50\textwidth}
  \includegraphics[width=1\textwidth]{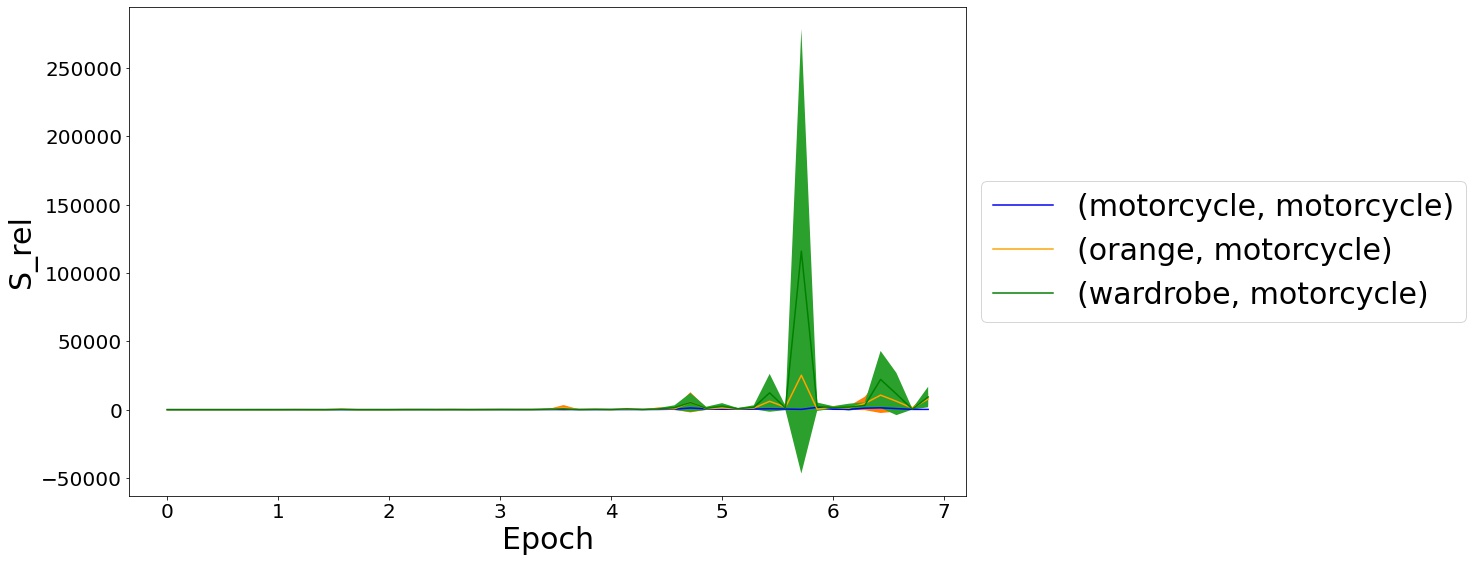}
  \caption{}
  \label{fig:13b}
\end{subfigure}
\medskip
\begin{subfigure}{0.50\textwidth}
  \includegraphics[width=\linewidth]{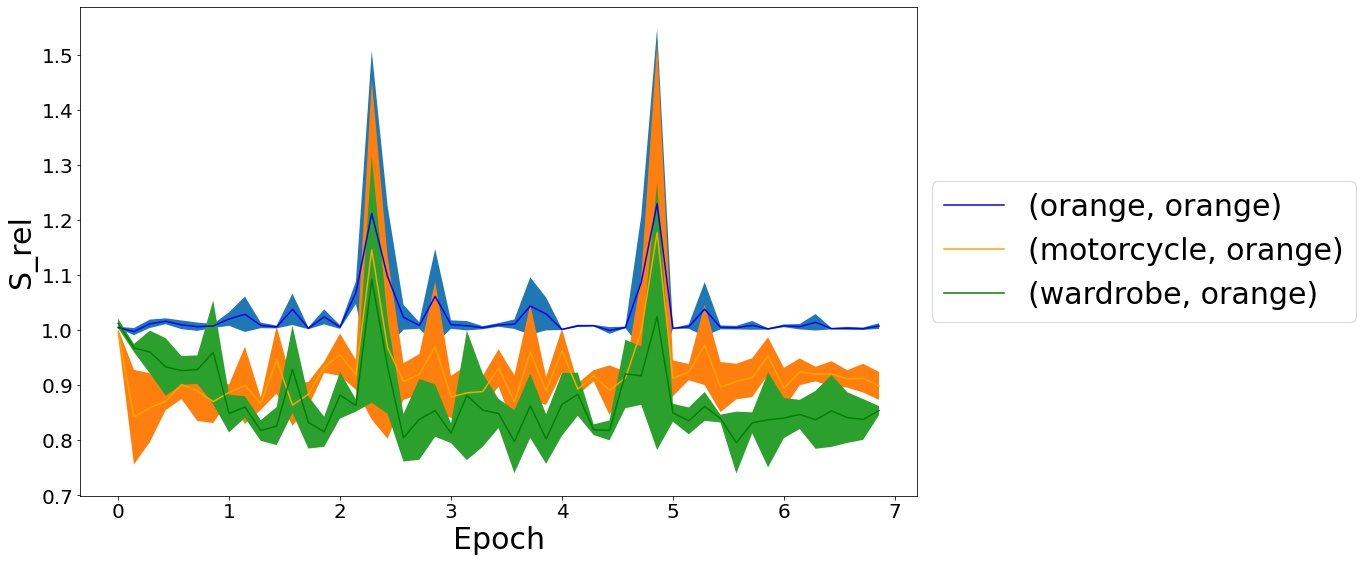}
  \caption{}
  \label{fig:13c}
\end{subfigure}\hfil 
\begin{subfigure}{0.50\textwidth}
  \includegraphics[width=\linewidth]{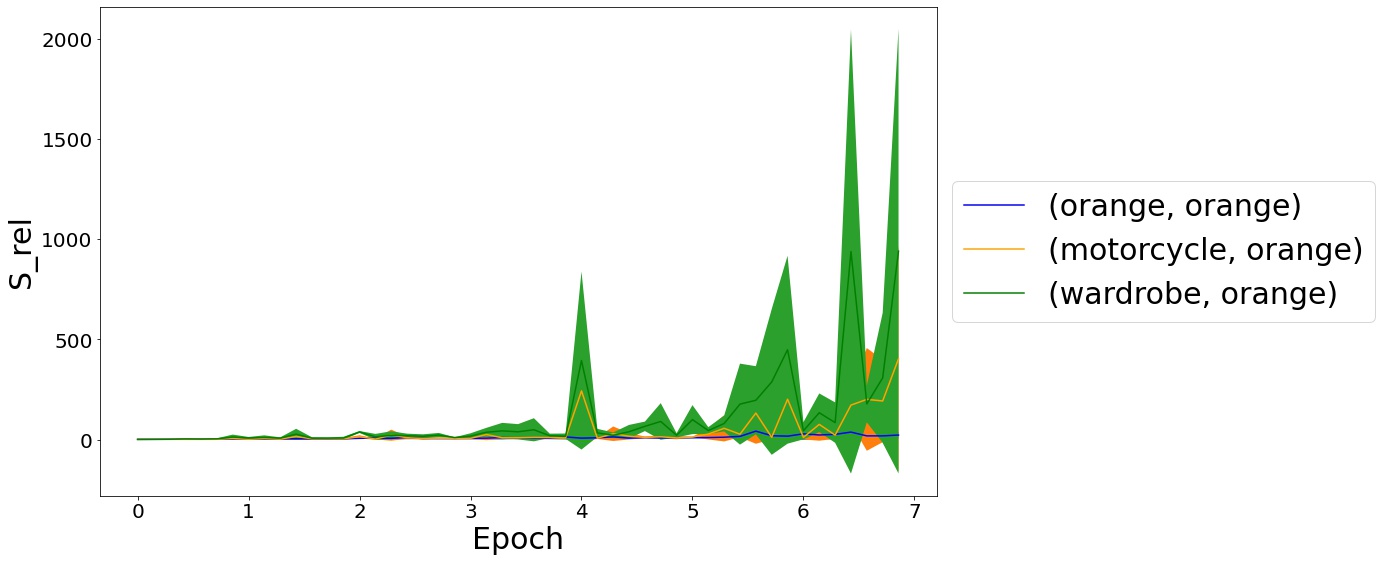}
  \caption{}
  \label{fig:13d}
\end{subfigure}
\medskip
\begin{subfigure}{0.50\textwidth}
  \includegraphics[width=\linewidth]{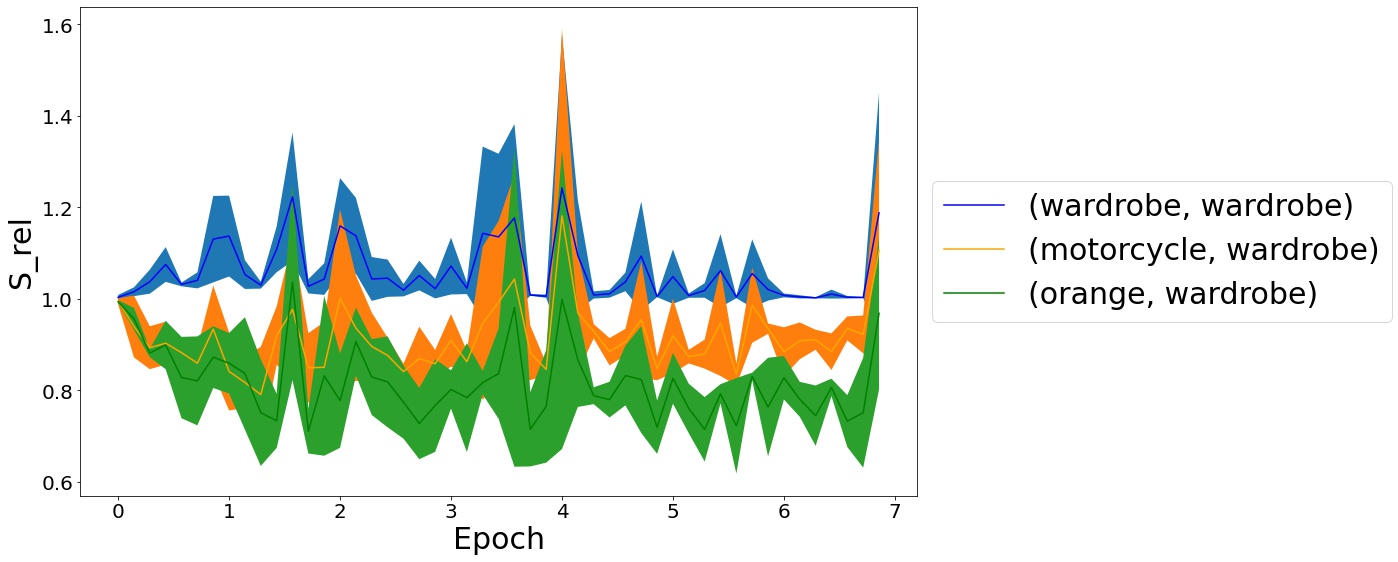}
  \caption{}
  \label{fig:13e}
\end{subfigure}\hfil 
\begin{subfigure}{0.50\textwidth}
  \includegraphics[width=\linewidth]{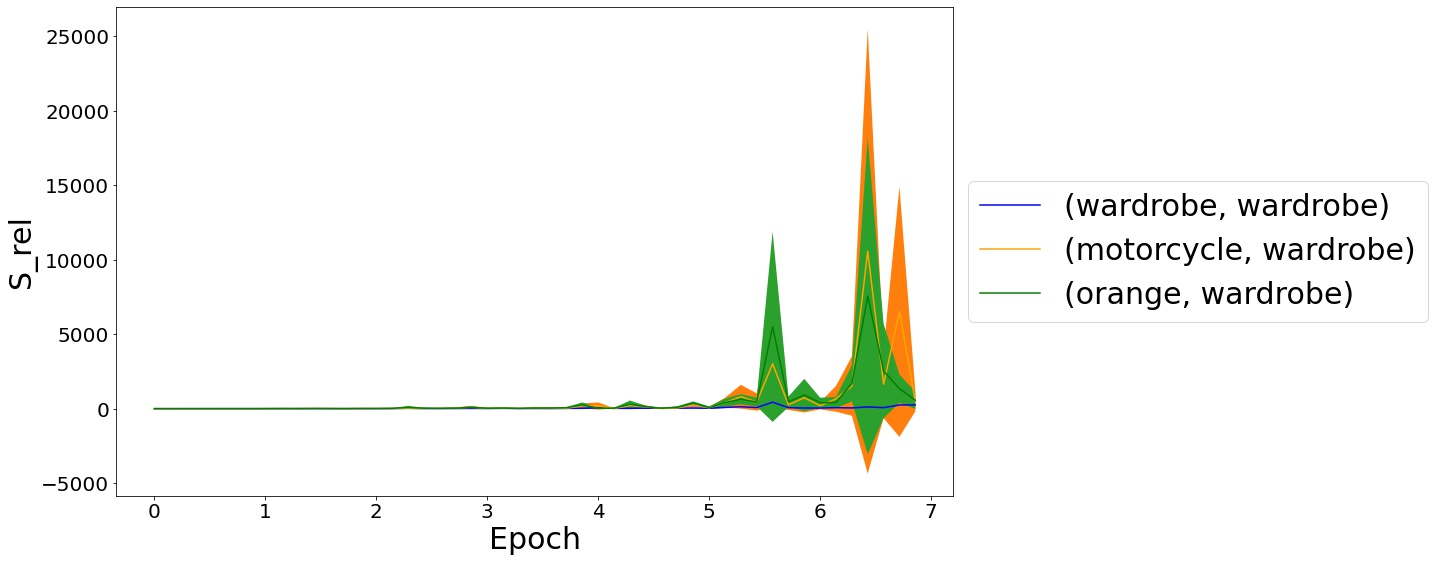}
  \caption{}
  \label{fig:13f}
\end{subfigure}
\caption{Time evolution of $S_{\rm rel}^{\rm k, smooth}$ (Definition \ref{def:srel_ce_k}) in figures (a), (c) and (e) and $S_{\rm rel}$ (Definition \ref{def:srel}) - averaged over data points, in figures (b), (d) and (f) for the two image classes being chosen from $\lbrace\text{motorcycle, orange, wardrobe} \rbrace$ with {\rm Class-$\x$} being \emph{motorcycle} in (a) and (b), \emph{orange} in (c) and (d), and \emph{wardrobe} in (e) and (f). Figures \ref{fig:loss_acc_dynamics_cifar100} shows the loss and accuracy dynamics for the experiment.}
\label{fig:srel_cifar100}
\end{figure}

\subsection{Other Experimental Details}
The experiments were run on the Colab GPU and an NVIDIA 11 GB GPU. The training time for each run (7 epochs) is less than an hour and the ResNet-18 (without batch norm) achieves $87, 78, 94 \%$ accuracies respectively on SVHN, CIFAR-10,100.

\begin{figure}[htbp]
    \centering 
\begin{subfigure}{0.50\textwidth}
  \includegraphics[width=1\textwidth]{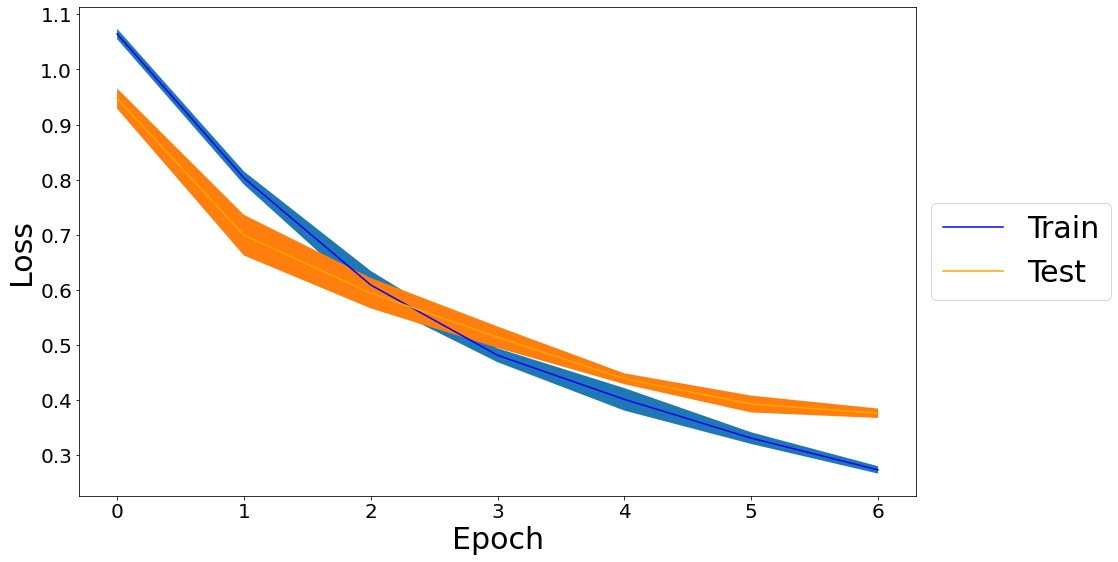}
  \caption{}
  \label{fig:14a}
\end{subfigure}\hfil 
\begin{subfigure}{0.50\textwidth}
  \includegraphics[width=1\textwidth]{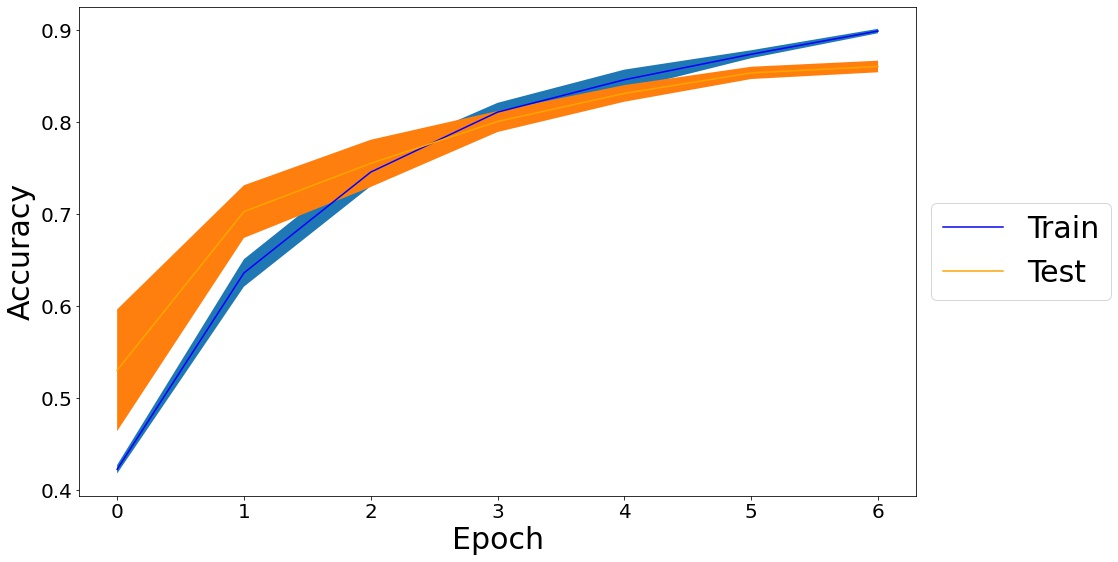}
  \caption{}
  \label{fig:14b}
\end{subfigure}
\caption{Dynamics of loss (figure (a)) and accuracy (figure (b)) during training on SVHN}
\label{fig:loss_acc_dynamics_svhn}
\end{figure}

\begin{figure}[htbp]
    \centering 
\begin{subfigure}{0.50\textwidth}
  \includegraphics[width=\linewidth]{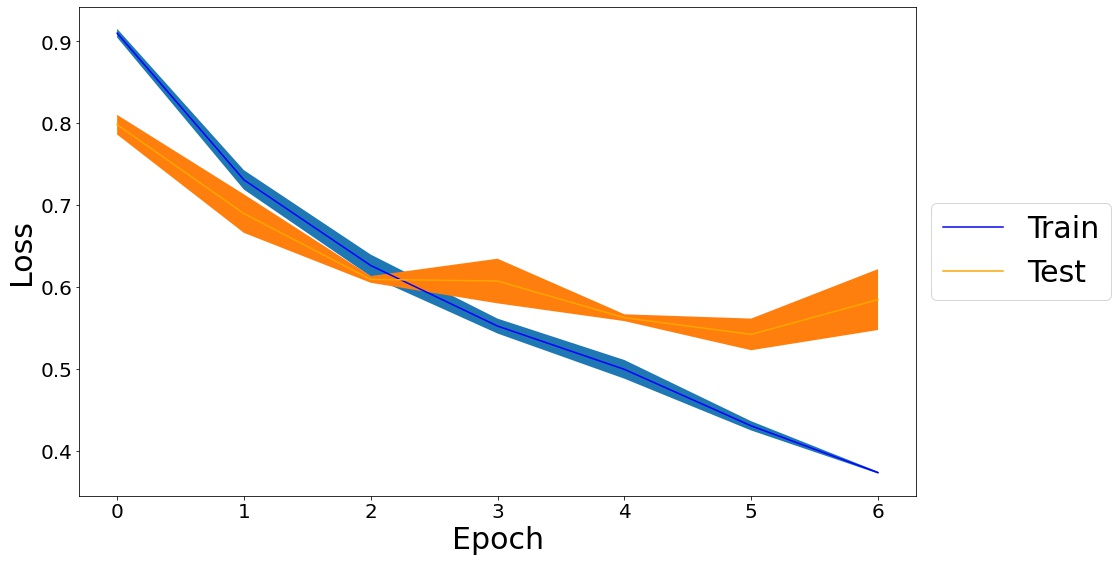}
  \caption{}
  \label{fig:15a}
\end{subfigure}\hfil 
\begin{subfigure}{0.50\textwidth}
  \includegraphics[width=\linewidth]{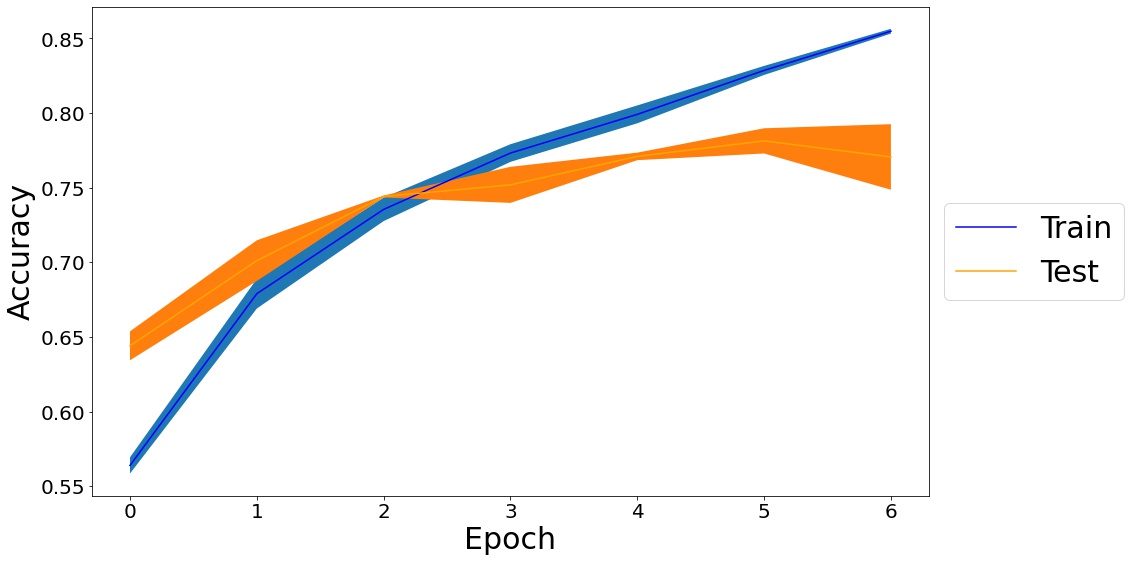}
  \caption{}
  \label{fig:15b}
\end{subfigure}
\caption{Dynamics of loss (figure (a)) and accuracy (figure (b)) during training on CIFAR-10}
\label{fig:loss_acc_dynamics_cifar10}
\end{figure}

\begin{figure}[htbp]
    \centering 
\begin{subfigure}{0.50\textwidth}
  \includegraphics[width=\linewidth]{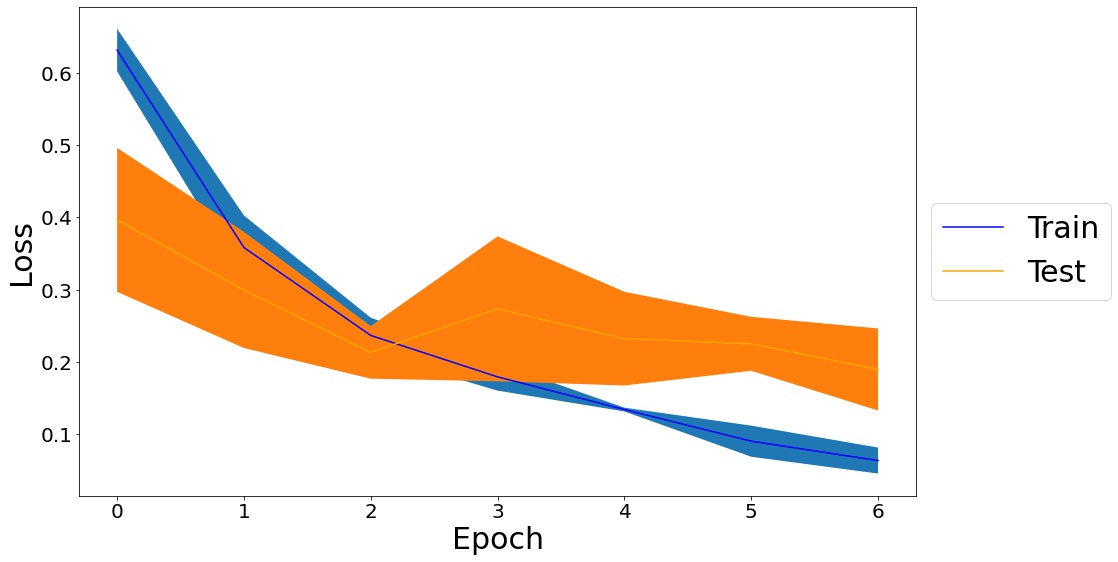}
  \caption{}
  \label{fig:16a}
\end{subfigure}\hfil 
\begin{subfigure}{0.50\textwidth}
  \includegraphics[width=\linewidth]{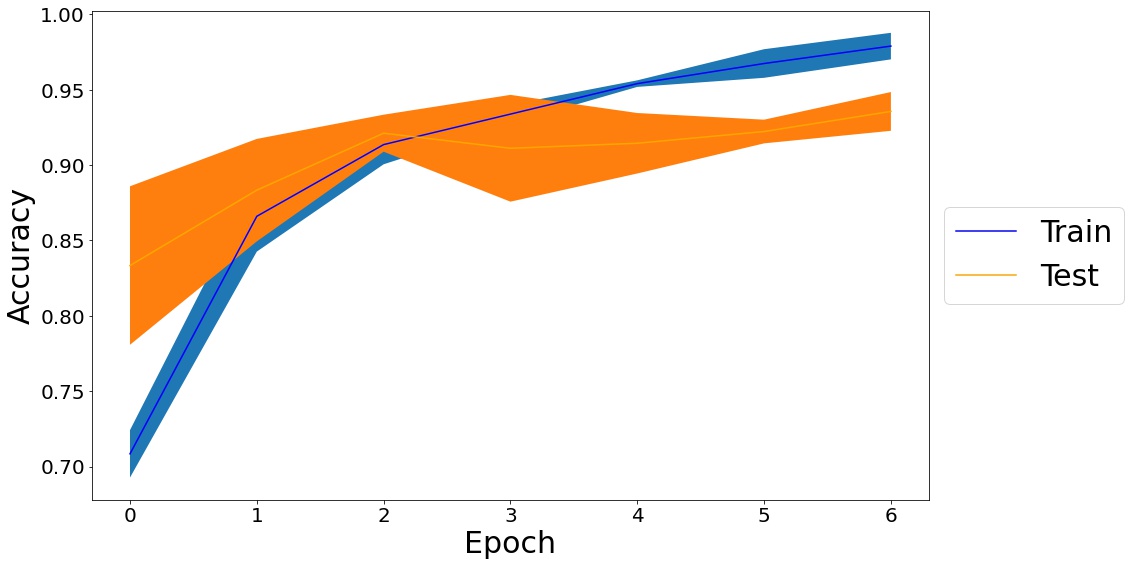}
  \caption{}
  \label{fig:16b}
\end{subfigure}
\caption{Dynamics of loss (figure (a)) and accuracy (figure (b)) during training on CIFAR-100}
\label{fig:loss_acc_dynamics_cifar100}
\end{figure}

\clearpage
\end{document}